\documentclass[final]{alt2026} 

\title[Relative Information Gain and Gaussian Process Regression]{Relative Information Gain and Gaussian Process Regression}
\usepackage[utf8]{inputenc}
\usepackage[T1]{fontenc}
\usepackage{microtype}
\usepackage{times}
\usepackage{hyperref}
\usepackage{mypackage}
\usepackage{xcolor}
\usepackage{float}

\newtheorem{assumption}{Assumption}

\newcommand{\bs}[1]{\boldsymbol{#1}}

\altauthor{%
 \Name{Hamish Flynn} \Email{hamishflynn.gm@gmail.com}\\
 \addr Universitat Pompeu Fabra, Barcelona, Spain
}

\begin{document}

\maketitle

\begin{abstract}%
The sample complexity of estimating or maximising an unknown function in a reproducing kernel Hilbert space is known to be linked to both the effective dimension and the information gain associated with the kernel. While the information gain has an attractive information-theoretic interpretation, the effective dimension typically results in better rates. We introduce a new quantity called the relative information gain, which measures the sensitivity of the information gain with respect to the observation noise. We show that the relative information gain smoothly interpolates between the effective dimension and the information gain, and that the relative information gain has the same growth rate as the effective dimension. In the second half of the paper, we prove a new PAC-Bayesian excess risk bound for Gaussian process regression. The relative information gain arises naturally from the complexity term in this PAC-Bayesian bound. We prove bounds on the relative information gain that depend on the spectral properties of the kernel. When these upper bounds are combined with our excess risk bound, we obtain minimax-optimal rates of convergence.
\end{abstract}


\section{Introduction}

We consider the model
\begin{equation}
y_i = f^{\star}(x_i) + \eps_i\,,\label{eqn:model}
\end{equation}
with inputs $x_1, \dots, x_n \in \gX$ and real-valued responses $y_1, \dots, y_n$. The target function $\fstar: \gX \to \sR$ is an unknown function in a reproducing kernel Hilbert space (RKHS) $\gH$ with reproducing kernel $k: \gX \times \gX \to \sR$. The noise variables are assumed to be independent, centred and $\sigma$-sub-Gaussian.

The model in \eqref{eqn:model} has been extensively studied in the regression setting, in which a sample $(x_i,y_i)_{i=1}^{n}$ is used to estimate either $\fstar$ or the vector of function values $\vfstar_n := [\fstar(x_1), \dots, \fstar(x_n)]^{\top}$ \citep{gyorfi2002distribution, tsybakov2009introduction}. Kernel ridge regression and Gaussian process regression are among the most popular approaches for the regression problem. For a regularisation parameter $\eta > 0$, the kernel ridge regression estimate is defined as
\begin{equation*}
\wh f := \argmin_{f \in \gH}\bigg\{\sum_{i=1}^{n}(f(x_i) - y_i)^2 + \frac{1}{\eta}\|f\|_{\gH}^2\bigg\}\,.
\end{equation*}
It is well-known that there is a closed-form expression for $\wh f$ \citep{scholkopf2002learning}, which is
\begin{equation*}
\wh f(x) = \vk_n^{\top}(x)(\mK_n + \tfrac{1}{\eta}\mI)^{-1}\vy_n\,.
\end{equation*}
Here, $\vk_n(x) := [k(x, x_1), \dots, k(x,x_n)]^{\top}$ is the vector of kernel comparisons between $x$ and the inputs $x_1, \dots, x_n$, $\mK_n := \{k(x_i, x_j)\}_{i,j=1}^{n}$ is the usual $n\times n$ kernel matrix and $\vy_n := [y_1, \dots, y_n]^{\top}$ is the response vector. Let us define
\begin{equation*}
\wh \vf_n := [\wh f(x_1), \dots, \wh f(x_n)]^{\top} = \mK_n(\mK_n + \tfrac{1}{\eta}\mI)^{-1}\vy_n
\end{equation*}
as the vector of fitted values. Since each fitted value is a linear combination of the responses $y_1, \dots, y_n$, kernel ridge regression is an example of a linear smoother. This means that we can write $\wh\vf_n = \mL_n\vy_n$. In this case, $\mL_n = \mK_n(\mK_n + \tfrac{1}{\eta}\mI)^{-1}$. The matrix $\mL_n$ is called the smoothing matrix, and its trace, denoted by $\mathrm{tr}(\mL_n)$, is called the effective degrees of freedom \citep{wasserman2006all}. We will revisit this quantity shortly.

Gaussian process regression is usually presented as a Bayesian method. It is assumed that $\fstar \sim \gG\gP(m(x), k(x, x^{\prime}))$ is a random draw from  a Gaussian process prior with mean function $m: \gX \to \sR$ and kernel (or covariance) function $k:\gX \times \gX \to \sR$. This means that for any fixed sequence of inputs $x_1, \dots, x_n$, the random vector $\vfstar_n$ is Gaussian, and in particular, $\vfstar_n \sim \gN(\vm_n, \mK_{n})$, where $\vm_n := [m(x_1), \dots, m(x_n)]^{\top}$. If the noise variables are Gaussian with variance $1/\eta$, then the Bayesian posterior (say $Q$) is also a Gaussian process \citep{williams2006gaussian}. In particular, if $\fstar$ is drawn from a zero-mean prior, i.e.\, $\fstar \sim \gG\gP(0, k(x, x^{\prime}))$, then
\begin{equation*}
Q = \gG\gP(\vk_n^{\top}(x)(\mK_n + \tfrac{1}{\eta}\mI)^{-1}\vy_n, k(x, x^{\prime}) - \vk_n^{\top}(x)(\mK_n + \tfrac{1}{\eta}\mI)^{-1}\vk_n(x^{\prime}))\,.
\end{equation*}
Note that the mean function of $Q$ is identical to the kernel ridge regression estimate. It turns out that the covariance function of $Q$, or rather, the covariance matrix of the marginal distribution of $Q$ at the points $x_1, \dots, x_n$, is closely related to the effective degrees of freedom $\mathrm{tr}(\mL_n)$ (cf. Lemma \ref{lem:cov_as_deff}). Depending on the context in which it appears, the parameter $\eta$ is given various names. Throughout the rest of the paper, we will refer to $\eta$ as a learning rate.

The model in \eqref{eqn:model} has also been studied in the bandit setting \citep{srinivas2010gaussian, valko2013finite}. In kernelised bandits (also known as Gaussian process bandits), $\fstar$ is called the reward function. The objective is to sequentially maximise the reward function by querying it at a sequence of points $x_1, x_2, \dots$ and receiving the random rewards $y_1, y_2, \dots$, which can be used to inform the selection of future query points. Since the reward function is initially unknown, it must be estimated and maximised simultaneously. For this reason, many kernelised bandit algorithms use kernel ridge regression or Gaussian process regression as a subroutine for estimation.

In both kernel regression and kernelised bandits, there are two widely used notions of complexity, called the effective dimension and the information gain, which are designed to characterise the sample complexity of estimating or maximising the function $\fstar$. For any $n \geq 1$ and any learning rate $\eta > 0$, the effective dimension \citep{zhang2005learning} is
\begin{equation*}
d_n(\eta) := \mathrm{tr}(\mK_n(\mK_n + \tfrac{1}{\eta}\mI)^{-1})\,.
\end{equation*}
Note that $d_n(\eta)$ is identical to the effective degrees of freedom of the kernel ridge regression estimate. For any $n \geq 1$ and any learning rate $\eta \geq 0$, the information gain is defined as
\begin{equation*}
\gamma_n(\eta) := \frac{1}{2}\log\det(\eta\mK_n + \mI)\,.
\end{equation*}
As its name would suggest, the information gain has an information-theoretic interpretation. If in \eqref{eqn:model}, $\fstar \sim \gG\gP(0, k(x, x^{\prime}))$ and $\eps_i \sim \gN(0, 1/\eta)$, then $\gamma_n(\eta)$ is equal to the mutual information $I(\vy_n; \fstar)$ between $\fstar$ and the response vector $\vy_n$. This version of the model in \eqref{eqn:model}, in which both $f^{\star}$ and $\epsilon_1, \dots, \epsilon_n$ are random, is a special case of a Gaussian channel (cf. Chapter 9 in \citealp{cover2006elements}). In the study of Gaussian channels, $\gamma_n(\eta)$ is an object of considerable interest.

It is known that effective dimension and the information gain are within logarithmic factors of each other \citep{calandriello2019gaussian, zenati2022efficient}. In particular, the growth-rate (in $n$) of the effective dimension is never larger than that of the information gain, whereas the growth-rate of the information gain can be larger than that of the effective dimension by a factor of $\log(n)$. Since both quantities can be used to characterise the sample complexity of estimating or maximising $\fstar$, it is natural to ask whether there are any other connections between them. One can also ask if there is a single quantity that has both the growth-rate of the effective dimension and an attractive information-theoretic interpretation like that of the information gain. We investigate these questions.

\subsection{Contributions}

We introduce a new quantity, called the relative information gain, which is the difference between the information gain at two different learning rates. As the smaller of the two learning rates is varied, a scaled version of the relative information gain smoothly interpolates between the effective dimension and the information gain, recovering each of them at the extremes. Moreover, the relative information gain matches the growth-rate of the effective dimension.

We demonstrate that the relative information gain is a reasonably natural notion of complexity for the model in \eqref{eqn:model}, and not just an artificial quantity that is designed to interpolate between the effective dimension and the information gain. In particular, we derive a localised PAC-Bayesian excess risk bound for Gaussian process regression (cf. Theorem \ref{thm:rig_excess_risk}) and we find that the relative information gain arises naturally from the complexity term in this PAC-Bayesian bound.

Finally, we show that if $k$ is a Mercer kernel, then the relative information gain can be upper bounded based on the rate at which its eigenvalues decay. When these upper bounds are combined with our excess risk bound in Theorem \ref{thm:rig_excess_risk}, we obtain minimax-optimal rates of convergence.

\subsection{Outline}

The rest of this paper is organised as follows. Section \ref{sec:related} describes some related work on rates of convergence for kernel ridge regression and Gaussian process regression, and on PAC-Bayesian bounds for Gaussian processes. In Section \ref{sec:deff_ig_rig}, we define the relative information gain and establish some of its properties. In Section 4, we state and sketch the proof of a PAC-Bayesian excess risk bound. In Section 5, we provide upper bounds on the relative information gain and then use them to obtain rates of convergence with explicit dependence on the sample size.

\section{Related Work}
\label{sec:related}

\subsection{Rates of Convergence for Kernel Ridge Regression and Gaussian Process Regression}

For the setting we consider, in which $\fstar$ is a fixed function in the RKHS $\gH$, the optimal rates of convergence for kernel ridge regression are well-understood. The rates of convergence that we obtain are the same as those in \citep{caponnetto2007optimal, steinwart2009optimal, dicker2015kernel, dicker2017kernel}. These works impose similar conditions on the eigenvalues and eigenfunctions of the kernel. The rates of convergence for Gaussian process regression are also well-understood in this setting. The rate at which the (Bayesian or Gibbs) Gaussian process posterior contracts has been the subject of intense study \citep{seeger2008information, castillo2008lower, van2008rates, van2011information, suzuki2012pac, castillo2014bayesian, pati2015optimal, nickl2017nonparametric}. Several of these results have been used to derive rates of convergence for the excess risk \citep{van2011information, suzuki2012pac}. In summary, it is already understood that both kernel ridge regression and Gaussian process regression can achieve the minimax optimal rate of convergence for the setting that we consider. Our work provides a new way to prove optimal excess risk bounds using the relative information gain.

\subsection{PAC-Bayesian Bounds for Gaussian Processes}

PAC-Bayesian bounds originate from work by \citet{shawe1997pac} and \citet{mcallester1998some}. We refer the reader to \citet{alquier2024user} or  \citet{hellstrom2025generalization} for a recent overview. Early work on PAC-Bayesian analysis of Gaussian processes focused on classification \citep{seeger2002pac, seeger2003bayesian}. More recently, PAC-Bayesian bounds were used to fit Gaussian process predictors for either classification or regression \citep{reeb2018learning}. The main focus of the aforementioned works was on obtaining numerically tight risk certificates, as opposed to obtaining the best rates of convergence. \citet{suzuki2012pac} derived PAC-Bayesian excess risk bounds for Gaussian process regression which match the optimal rate of convergence under certain conditions. Despite having a similar rate of convergence, these bounds are quite different to ours. In particular, the rate of convergence is determined by the concentration function used in \citet{van2008rates, van2011information}, as opposed to a notion of effective dimension or information gain. \citet{alquier2020concentration} used PAC-Bayesian bounds to derive rates of convergence for Variational Bayes approximations of Gibbs distributions (such as Gaussian processes). When applied to nonparametric regression over Sobolev ellipsoids, the rate of convergence is optimal up to logarithmic factors. More recent work \citep{khribch2024convergence} has shown that optimal rates of convergence for a range of statistical models (including the Gaussian sequence model) can be obtained via localised PAC-Bayesian bounds. \citet{khribch2024convergence} highlight the connection between localised PAC-Bayesian bounds and mutual information bounds, although the relative information gain does not make an appearance.

\section{Effective Dimension, Information Gain and Relative Information Gain}
\label{sec:deff_ig_rig}

We describe a connection between the effective dimension and the information gain, and we introduce the relative information gain. First, we notice that both the effective dimension and the information gain can be expressed in terms of the eigenvalues $(\lambda_i)_{i=1}^{n}$ of $\mK_n$. In particular, by Lemma \ref{lem:deff_eqdef} and Lemma \ref{lem:ig_eqdef},
\begin{equation}
d_n(\eta) = \sum_{i=1}^{n}\frac{\eta\lambda_i}{1 + \eta\lambda_i}\,, \qquad \gamma_n(\eta) = \frac{1}{2}\sum_{i=1}^{n}\log(1 + \eta\lambda_i)\,.\label{eqn:deff_ig_eig}
\end{equation}
Starting from these expressions, one can easily verify that the effective dimension is related to the derivative of the information gain.
\begin{proposition}
For all $n \geq 1$ and $\eta \geq 0$,
\begin{equation*}
d_n(\eta) = 2\eta\gamma_n^{\prime}(\eta)\,.
\end{equation*}
\label{pro:deff_diff}
\end{proposition}\vspace{-7mm}
This identity is equivalent to a known result for Gaussian channels, which states that the derivative of the information gain with respect to the signal-to-noise ratio is equal to half the minimum mean squared error \citep{guo2005mutual}. In our formulation, $\eta$ is the signal-to-noise ratio, and it turns out that the minimum mean squared error is equal to $d_n(\eta)/\eta$ (cf.\, Appendix \ref{sec:mmse_deff}). Thus Proposition 1 can be obtained from Theorem 1 in \citet{guo2005mutual}.

The identity in Proposition \ref{pro:deff_diff} gives us an information-theoretic interpretation of the effective dimension. Namely, the effective dimension is a measure of how sensitive the information gain is to the learning rate (or the variance of the noise in the responses). In addition, Proposition \ref{pro:deff_diff} suggests that we can obtain a reasonable approximation of the effective dimension by taking the difference between the information gain at two different learning rates. For a sample size $n \geq 1$, and two learning rates $\eta > \beta \geq 0$, we define the relative information gain as
\begin{equation*}
\gamma_n(\eta, \beta) := \gamma_n(\eta) - \gamma_n(\beta)\,.
\end{equation*}
The relative information gain can be interpreted as the additional information that would be gained about $\fstar$ if the variance of the noise in the responses was reduced from $1/\beta$ to $1/\eta$. The bottom toast inequality (the bottom half of the sandwich inequality) in Proposition 5 of \citet{calandriello2019gaussian} states that $d_n(\eta)$ is bounded by $2\gamma_n(\eta)$. We will now show that a scaled version of the relative information gain smoothly interpolates between these two quantities. Using Proposition \ref{pro:deff_diff}, and the definition of the derivative, we can express the effective dimension as (the limit of) a scaled version of the relative information gain. In particular, for any learning rates $\eta > \beta > 0$,
\begin{equation}
d_n(\eta) = \lim_{\eta_0 \to \eta^-}\frac{2\eta}{\eta-\eta_0}\gamma_n(\eta, \eta_0) \approx \frac{2\eta}{\eta-\beta}\gamma_n(\eta, \beta)\,.\label{eqn:rig_limit}
\end{equation}
If $\beta=0$, then the right-hand side of \eqref{eqn:rig_limit} is equal to twice the information gain, although the approximation of the limit is rather crude in this case. This establishes that the scaled information gain recovers both the effective dimension and the information gain as special cases. The following proposition shows that these are the two extreme cases.
\begin{proposition}
For all $n \geq 1$ and $\eta > \beta \geq 0$,
\begin{equation*}
d_{n}(\eta) \leq \frac{2\eta}{\eta - \beta}\gamma_n(\eta, \beta) \leq 2\gamma_n(\eta)\,.
\end{equation*}
The first inequality is sharp in the limit as $\beta$ tends to $\eta$ from below and the second inequality is sharp when $\beta = 0$.
\label{pro:interp}
\end{proposition}
\begin{proof}
The sharpness of each inequality follows from \eqref{eqn:rig_limit}. For any $\lambda \geq 0$, the mapping $\beta \mapsto \frac{1}{\eta - \beta}\log\frac{1 + \eta\lambda}{1 + \beta\lambda}$ is decreasing on $[0, \eta)$ (cf. Lemma \ref{lem:rig_decr}). From this and \eqref{eqn:deff_ig_eig}, it follows that $\beta \mapsto \frac{2\eta}{\eta - \beta}\gamma_n(\eta, \beta)$ is also decreasing on $[0, \eta)$. Thus the scaled information gain is bounded between $d_n(\eta)$ and $2\gamma_n(\eta)$ for all $\beta$ in the interval $[0, \eta)$.
\end{proof}
The full sandwich inequality in Proposition 5 of \citet{calandriello2019gaussian} shows that the effective dimension and the information gain are within logarithmic factors of each other. In particular,
\begin{equation*}
d_{n}(\eta) \leq 2\gamma_n(\eta) \leq (1 + \log(1 + \eta\lambda_{\max}))d_n(\eta)\,,
\end{equation*}
where $\lambda_{\max} = \max_{i}\lambda_i$. Since $\lambda_{\max}$ is at most of order $n$, this implies that $\gamma_n(\eta) = \gO(d_n(\eta)\log(n))$. We prove that the effective dimension and the relative information gain can form a thinner sandwich.
\begin{proposition}
For all $n \geq 1$ and $\eta > \beta > 0$,
\begin{equation*}
d_{n}(\eta) \leq \frac{2\eta}{\eta - \beta}\gamma_n(\eta, \beta) \leq \frac{\eta}{\beta}d_{n}(\eta)
\end{equation*}
\label{pro:sandwich}
\end{proposition}
\begin{proof}
The first inequality follows from Proposition \ref{pro:interp}. Using the inequality $\log(1 + x) \leq x$ for $x \geq 0$, we obtain
\begin{equation*}
\frac{\eta}{\eta - \beta}\sum_{i=1}^{n}\log\frac{1 + \eta\lambda_i}{1 + \beta\lambda_i} \leq \sum_{i=1}^{n}\frac{\eta\lambda_i}{1 + \beta\lambda_i} = \sum_{i=1}^{n}\frac{1 + \eta\lambda_i}{1 + \beta\lambda_i}\frac{\eta\lambda_i}{1 + \eta\lambda_i} \leq \frac{\eta}{\beta}\sum_{i=1}\frac{\eta\lambda_i}{1 + \eta\lambda_i}\,.
\end{equation*}
The second inequality now follows from \eqref{eqn:deff_ig_eig}.
\end{proof}
Proposition \ref{pro:sandwich} tells us that the effective dimension and the relative information gain have the same growth-rate whenever $\beta = c\eta$, for some constant $c \in (0,1)$.

\subsection{Illustrative Examples}
\label{sec:examples}

We demonstrate that the information gain and the relative information gain really can have different growth rates. First, we consider a simple example in which we can easily verify that $\gamma_n(\eta, \beta) = \gO(1)$ and $\gamma_n(\eta) = \Omega(\log(n))$. We assume that the inputs $x_1, \dots, x_n \in \gX$ are elements of a finite set $\gX$ with cardinality $|\gX| = k < n$, and that the kernel function satisfies $k(x_i, x_i) = 1$ for all $i \in [n]$. First, we upper bound the effective dimension. Suppose that $\mK_n$ has $m$ non-zero eigenvalues, and let us assume that it is the first $m$ eigenvalues that are non-zero. Then
\begin{equation*}
d_n(\eta) = \sum_{i=1}^{m}\frac{\eta\lambda_i}{1 + \eta\lambda_i} \leq \sum_{i=1}^{m}\frac{\eta\lambda_i}{\eta\lambda_i} = m\,.
\end{equation*}
Since $|\gX| \leq k$, the rank of the kernel matrix $\mK_n$ is at most $k$, which means that $d_n(\eta) \leq k$. By Proposition \ref{pro:sandwich}, this also means that $\gamma_n(\eta, \beta) \leq \frac{\eta-\beta}{2\beta}k$. If for some $c \in (0, 1)$, $\beta = c\eta$, then $\gamma_n(\eta, \beta) = \gO(1)$. Next, we lower bound the information gain. Since every term in the expansion of $\prod_{i=1}^{n}(1 + \eta\lambda_i)$ is non-negative,
\begin{equation*}
\gamma_n(\eta) = \frac{1}{2}\log\prod_{i=1}^{n}(1 + \eta\lambda_i) \geq \frac{1}{2}\log\bigg(1 + \sum_{i=1}^{n}\eta\lambda_i\bigg) = \frac{1}{2}\log(1 + \eta n)\,.
\end{equation*}
Here, we used the fact that since $k(x_i, x_i) = 1$ for all $i \in [n]$, $\mathrm{tr}(\mK_n) = n$. We conclude that whenever $\eta$ does not depend on $n$, $\gamma_n(\eta) = \Omega(\log(n))$.

Next, we consider a somewhat less artificial example, and estimate the growth-rates of the information gain and the relative information gain via simulation. This time, the inputs $x_1, \dots, x_n$ are drawn independently and uniformly at random from the interval $[0, 1]$, the kernel function is the squared exponential kernel $k(x, y) = \exp(-(x-y)^2/2)$ and we take $\eta = 1$ and $\beta = 1/2$. For several values of $n$ between $1$ and $3000$, and 20 random draws of $x_1, \dots, x_n$, we record the values of all three quantities that appear in Proposition \ref{pro:interp}. Namely, we record $d_n(\eta)$, $\frac{2\eta}{\eta-\beta}\gamma_n(\eta, \beta)$ and $2\gamma_n(\eta)$. We plot the average values of these three quantities in Figure \ref{fig:growth}, and obtain the estimates
\begin{align*}
\sE[\gamma_n(1)] \approx 0.267\cdot\log^{8/5}(n) + 0.347\,, ~~~\sE[\gamma_n(1, 1/2)] \approx 0.113\cdot\log(n) + 0.144\,.
\end{align*}
We see that the information gain grows at  a rate that is faster by a factor of approximately $\log^{3/5}(n)$.

\begin{figure}
\centering
\includegraphics[width=0.9\textwidth]{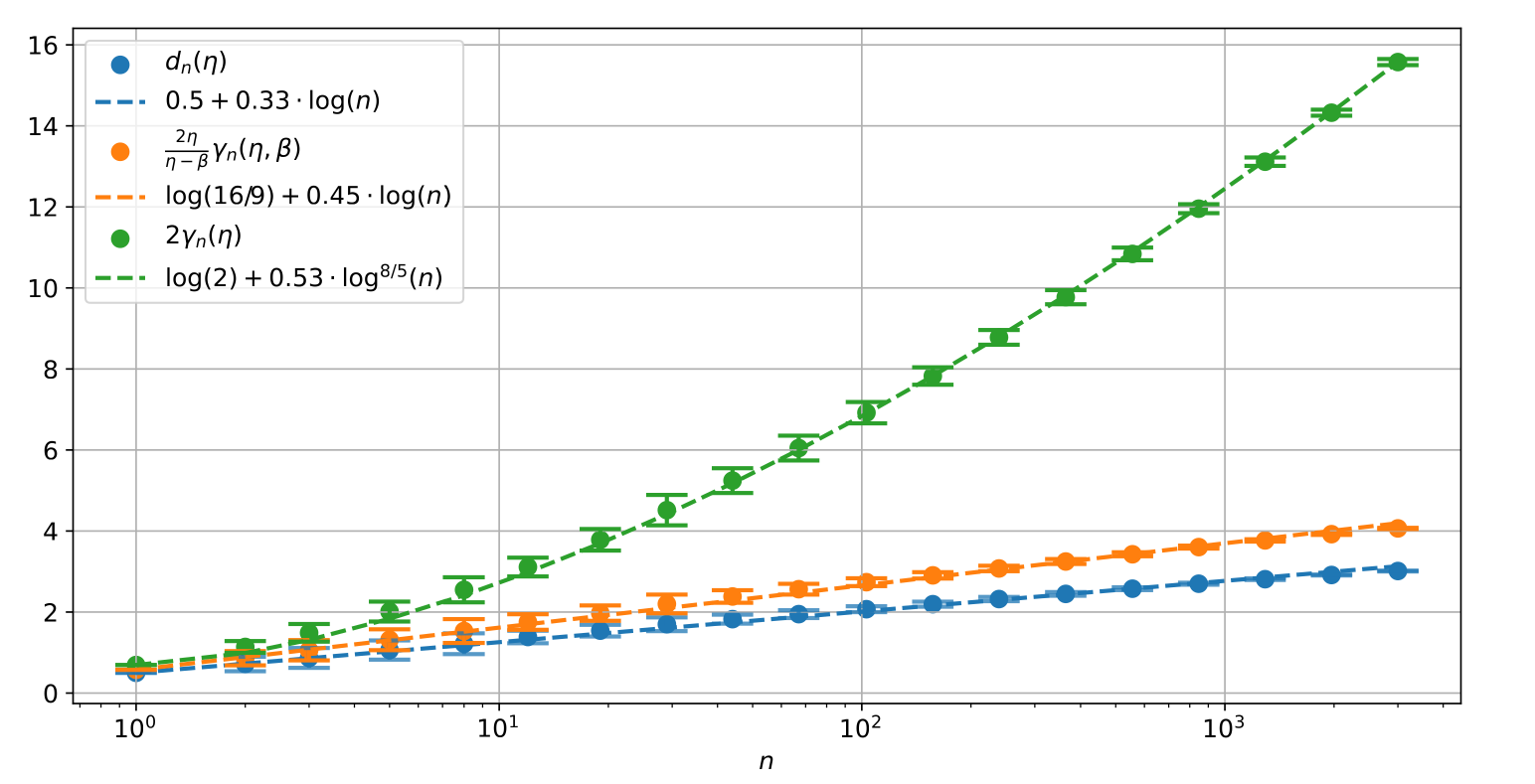}
\caption{Estimates of the effective dimension, the information gain and the relative information gain. The inputs are drawn uniformly at random from $[0,1]$, the kernel function is $k(x,y) = \exp(-(x-y)^2/2)$ and the learning rates are $\eta = 1$ and $\beta = 1/2$. We plot the average value plus/minus three times the standard deviation over 20 repetitions.}
\label{fig:growth}
\end{figure}

\section{Excess Risk Bounds for Gaussian Process Regression}

We demonstrate that the relative information gain arises naturally from a localised PAC-Bayesian bound for Gaussian process regression. We use the zero-mean prior $P_{\alpha} = \gG\gP(0, \alpha k(x, x^{\prime}))$ with a scale parameter $\alpha > 0$. For a learning rate $\eta > 0$, we will consider Gibbs distributions $Q_{n,\eta,\alpha}$ defined by
\begin{equation*}
\dd{}{Q_{n, \eta, \alpha}}{P_{\alpha}}(f) \propto \exp\bigg(-\eta\sum_{i=1}^{n}(f(x_i) - y_i)^2\bigg)\,.
\end{equation*}
It is known that $Q_{n, \eta, \alpha}$ is also a Gaussian process \citep{williams2006gaussian}. In particular,
\begin{equation*}
Q_{n, \eta, \alpha} = \gG\gP(\vk_n^{\top}(x)(\mK_n + \tfrac{1}{2\eta\alpha}\mI)^{-1}\vy_n, \alpha k(x, x^{\prime}) - \alpha \vk_n^{\top}(x)(\mK_n + \tfrac{1}{2\eta\alpha}\mI)^{-1}\vk_n(x^{\prime}))\,.
\end{equation*}
Note that if we set the learning rate to $\eta = 1/(2\sigma^2)$, then $Q_{n, \eta, \alpha}$ coincides with the Bayesian posterior for the model in \eqref{eqn:model}. We consider the problem of regression with fixed design, which means that the inputs $x_1, \dots, x_n$ are deterministic and our objective is to estimate the function values $\fstar(x_1), \dots, \fstar(x_n)$. In this setting, we only need to consider the marginal distributions of $Q_{n, \eta, \alpha}$ and $P_{\alpha}$ on the inputs $x_1, \dots, x_n$, which are $n$-dimensional Gaussians. By a small abuse of notation, we will also use $Q_{n,\eta,\alpha}$ and $P_{\alpha}$ to refer to these marginal distributions. It is easy to check that $P_{\alpha} = \gN(0, \alpha \mK_n)$ and $Q_{n, \eta,\alpha} = \gN(\vm_{n,\eta,\alpha}, \mK_{n,\eta,\alpha})$, where
\begin{align*}
\vm_{n,\eta,\alpha} = \mK_n(\mK_n + \tfrac{1}{2\eta\alpha}\mI)^{-1}\vy_n\,, \quad \mK_{n,\eta,\alpha} = \alpha\mK_n - \alpha\mK_n(\mK_n + \tfrac{1}{2\eta\alpha}\mI)^{-1}\mK_n\,.
\end{align*}
We assume that the noise variables $\eps_1, \dots, \eps_n$ are independent and $\sigma$-sub-Gaussian.
\begin{assumption}
The noise variables $\eps_1, \dots, \eps_n$ are independent. For all $i \in \{1, \dots, n\}$ and all $\eta \in \sR$, $\sE[\exp(\eta\eps_i)] \leq \exp(\sigma^2\eta^2/2)$.
\label{ass:subG}
\end{assumption}
For simplicity, we also assume that the inputs $x_1, \dots, x_n$ and the kernel $k$ are such that the kernel matrix $\mK_n$ is strictly positive-definite. This assumption is quite mild, since it is satisfied by typical kernels whenever $x_1, \dots, x_n$ are distinct. It is used to prove some of the auxiliary lemmas in Appendix \ref{sec:aux_pac_bayes}, but we expect that if one has the desire, this assumption can be dropped.
\begin{assumption}
The kernel matrix $\mK_n$ is strictly positive-definite.
\label{ass:pd}
\end{assumption}
In what follows, let us use $\vg := [g_1, \dots, g_n]^{\top}$ to denote an arbitrary vector (of function values) in $\sR^n$. We define the \emph{excess risk} $R_n: \sR^n \to \sR$ and the \emph{empirical risk} $r_n: \sR^n \to \sR$ as
\begin{equation*}
R_n(\vg) := \frac{1}{n}\|\vg - \vfstar_n\|_2^2\,, \quad r_n(\vg) := \frac{1}{n}\|\vg - \vy_n\|_2^2\,.
\end{equation*}
We want to upper bound the average excess risk $\int R_n(\vg)\mathrm{d}Q_{n,\eta,\alpha}(\vg)$ of $Q_{n,\eta,\alpha}$. Using a PAC-Bayesian bound for sums of independent random variables (cf. Proposition 2.1 in \citealp{catoni2017dimension}), one can derive the following PAC-Bayesian bound. For any $n \geq 1$, any $\delta \in (0, 1]$, any $\eta \in (0, \frac{1}{2\sigma^2})$ and any $\alpha > 0$, with probability at least $1 - \delta$,
\begin{equation}
\int R_n(\vg)\mathrm{d}Q_{n,\eta,\alpha}(\vg) \leq \inf_{Q}\left\{\frac{\int\eta(r_n(\vg) - r_n(\vfstar_n))\mathrm{d}Q(\vg)}{\eta - 2\sigma^2\eta^2} + \frac{\KL(Q||P_{\alpha}) + \log\tfrac{1}{\delta}}{n(\eta - 2\sigma^2\eta^2)}\right\}\,.\label{eqn:rubbish_bound}
\end{equation}
We consider this to be a fairly standard PAC-Bayesian bound for the excess risk. Indeed, one can obtain a bound similar to the one in \eqref{eqn:rubbish_bound} from Theorem 2.1 in \citet{alquier2020concentration}. However, this bound is not completely satisfactory. Lemma \ref{lem:donsker} tells us that the infimum on the right-hand side is achieved when $Q = Q_{n,\eta,\alpha}$. For this choice of $Q$, the KL divergence $\KL(Q_{n,\eta,\alpha}||P_{\alpha})$ contains the log-determinant term $\frac{1}{2}\log(\det(\alpha\mK_n)/\det(\mK_{n,\eta,\alpha}))$. Due to Corollary \ref{cor:inv_bollocks},
\begin{equation*}
\frac{1}{2}\log\frac{\det(\alpha\mK_n)}{\det(\mK_{n,\eta,\alpha})} = \frac{1}{2}\log\det(2\eta\alpha\mK_n + \mI) = \gamma_n(2\eta\alpha)\,.
\end{equation*}
Therefore, one would expect that the bound in \eqref{eqn:rubbish_bound} is at best of the order $\gamma_n(2\eta\alpha)/n$. In fact, one can show that this bound is of the order $[\gamma_n(2\eta\alpha) + \|\fstar\|_{\gH}^2/\alpha]/n$. Sadly, plugging in the bounds on the information gain from \citet{vakili2021information} and then optimising the value of $\alpha$ does not yield minimax-optimal rates of convergence. The extra (up to) logarithmic factor in the growth rate of $\gamma_n(\eta)$ (compared to $d_n(\eta)$ or $\gamma_n(\eta, \beta)$) results in unnecessary log factors in the rate of convergence. To fix this, we use Catoni's localisation technique \citep{catoni2007pac}, which allows one to prove PAC-Bayesian bounds with data-dependent priors. In Proposition \ref{pro:excess_risk}, the prior is the Gibbs distribution $Q_{n,\beta,\alpha}$, with a learning rate $\beta < \eta$.
\begin{proposition}
Suppose that Assumption \ref{ass:subG} is satisfied. For any $n \geq 1$, any $\delta \in (0, 1]$, any $\eta \in (0, \frac{1}{2\sigma^2})$, any $\beta > 0$ such that $\eta - 2\eta^2\sigma^2 - \beta - 2\beta^2\sigma^2 > 0$ and any $\alpha > 0$, w.p. at least $1 - \delta$,
\begin{equation*}
\int R_n(\vg)\mathrm{d}Q_{n,\eta,\alpha}(\vg) \leq \inf_{Q}\left\{\frac{\int(\eta-\beta)(r_n(\vg) - r_n(\vfstar_n))\mathrm{d}Q(\vg)}{\eta - 2\sigma^2\eta^2 - \beta - 2\sigma^2\beta^2} + \frac{\KL(Q||Q_{n,\beta,\alpha}) + 2\log\tfrac{1}{\delta}}{n(\eta - 2\sigma^2\eta^2 - \beta - 2\sigma^2\beta^2)}\right\}\,.
\end{equation*}
\label{pro:excess_risk}
\end{proposition}
Using Lemma \ref{lem:donsker} again, it can be shown that the infimum on the right-hand side is still achieved when $Q = Q_{n,\eta,\alpha}$ (cf. Appendix \ref{sec:pb_excess_risk}). This time, the KL divergence $\KL(Q_{n,\eta,\alpha}||Q_{n,\beta,\alpha})$ contains the log-determinant term $\frac{1}{2}\log(\det(\mK_{n,\beta,\alpha}))/\det(\mK_{n,\eta,\alpha}))$. Using Corollary \ref{cor:inv_bollocks} again,
\begin{equation*}
\frac{1}{2}\log\frac{\det(\mK_{n,\beta,\alpha}))}{\det(\mK_{n,\eta,\alpha})} = \frac{1}{2}\log\det\frac{2\eta\alpha\mK_n + \mI}{2\beta\alpha\mK_n + \mI} = \gamma_n(2\eta\alpha, 2\beta\alpha)\,.
\end{equation*}
One might expect that the bound in Proposition \ref{pro:excess_risk} is of the order $[\gamma_n(2\eta\alpha, 2\beta\alpha) + \|\fstar\|_{\gH}^2/\alpha]/n$. We will see shortly that this is indeed the case. We sketch the main ideas of the proof of Proposition \ref{pro:excess_risk} here, and refer the interested reader to Appendix \ref{sec:pac_bayes}. A key idea is to use a prior that assigns higher probability to functions (or vectors of function values) for which the excess risk is small. We define the distribution-dependent prior $\overline{Q}_{n,\beta,\alpha}$ by
\begin{equation*}
\dd{}{\overline{Q}_{n,\beta,\alpha}}{P_{\alpha}}(f) \propto \exp\bigg(-(\beta + 2\sigma^2\beta^2)\sum_{i=1}^{n}(f(x_i) - \fstar(x_i))^2\bigg)\,.
\end{equation*}
Note that $\overline{Q}_{n,\beta,\alpha}$ depends on the distribution of $y_1, \dots, y_n$ through $x_1, \dots, x_n$ and $\fstar$, but it does not depend on the random draw of $y_1, \dots, y_n$, and is therefore a valid prior. It is not important for the proof, but one can also notice that $\overline{Q}_{n,\beta,\alpha}$ is another Gaussian process and that its marginal distribution at the inputs $x_1, \dots, x_n$ is
\begin{equation*}
\overline{Q}_{n,\beta,\alpha} = \gN(\mK_n(\mK_n + \tfrac{1}{2\alpha(\beta + 2\sigma^2\beta^2)}\mI)^{-1}\vfstar_n, \alpha\mK_n - \alpha\mK_n(\mK_n + \tfrac{1}{2\alpha(\beta + 2\sigma^2\beta^2)}\mI)^{-1}\mK_n)\,.
\end{equation*}
In the bound we want to prove, the prior is $Q_{n,\beta,\alpha}$. It turns out that for all $Q \in \Delta(\sR^n)$, the difference between $\KL(Q||\overline{Q}_{n,\beta,\alpha})$ and $\KL(Q||Q_{n,\beta,\alpha})$ can be upper bounded with high probability. This is made possible by the following concentration inequality, which is proved using Assumption \ref{ass:subG} (cf. Proposition \ref{pro:pb_indep}). For any $\eta \in \sR$ and any $\delta \in (0, 1]$, with probability at least $1 - \delta$,
\begin{equation*}
\sup_{Q}\bigg\{\int\eta(r_n(\vfstar_n) - r_n(\vg)) + (\eta - 2\sigma^2\eta^2)R_n(\vg)\mathrm{d}Q(\vg) - \KL(Q||\overline{Q}_{n,\beta,\alpha})\bigg\} \leq \log\tfrac{1}{\delta}\,.
\end{equation*}
One can then show that, with probability at least $1 - \delta$, for all $Q \in \Delta(\sR^n)$,
\begin{equation*}
-\KL(Q||Q_{n,\beta,\alpha}) - \int\beta(r_n(\vfstar_n) - r_n(\vg)) + (\beta + 2\sigma^2\beta^2)R_n(\vg)\mathrm{d}Q(\vg) \leq -\KL(Q||\overline{Q}_{n,\beta,\alpha}) + \log\tfrac{1}{\delta}\,.
\end{equation*}
By the union bound, these inequalities can be used in succession to deduce that, with probability at least $1 - \delta$,
\begin{align*}
\sup_{Q}&\bigg\{\int(\eta-\beta)(r_n(\vfstar_n) - r_n(\vg)) + (\eta - 2\sigma^2\eta^2 - \beta - 2\sigma^2\beta^2)R_n(\vg)\mathrm{d}Q(\vg) - \KL(Q||Q_{n,\beta,\alpha})\bigg\}\\
&\leq \sup_{Q}\bigg\{\int\eta(r_n(\vfstar_n) - r_n(\vg)) + (\eta - 2\sigma^2\eta^2)R_n(\vg)\mathrm{d}Q(\vg) - \KL(Q||\overline{Q}_{n,\beta,\alpha})\bigg\} + \log\tfrac{2}{\delta}\\
&\leq 2\log\tfrac{2}{\delta}\,.
\end{align*}
This inequality can be rearranged into the one in Proposition \ref{pro:excess_risk}. The $2\log(2/\delta)$ term can be improved to $2\log(1/\delta)$ by replacing the union bound with Cauchy-Schwarz (cf. Appendix \ref{sec:pb_indep}). Using Proposition \ref{pro:excess_risk}, we obtain the main result of this section.
\begin{theorem}
Suppose that Assumption \ref{ass:subG} and Assumption \ref{ass:pd} are satisfied. For any $n \geq 1$, any $\delta \in (0, 1]$, any $\eta \in (0, \frac{1}{2\sigma^2})$, any $\beta > 0$ such that $\eta - 2\eta^2\sigma^2 - \beta - 2\beta^2\sigma^2 > 0$ and any $\alpha > 0$, with probability at least $1 - \delta$,
\begin{equation*}
\int R_n(\vg)\mathrm{d}Q_{n, \eta, \alpha}(\vg) \leq \frac{2\gamma_n(2\eta\alpha, 2\beta\alpha) + \frac{1}{2\alpha}\|\fstar\|_{\gH}^2 + 2\log\tfrac{1}{\delta}}{n(\eta - 2\sigma^2\eta^2 - \beta - 2\sigma^2\beta^2)}\,.
\end{equation*}
\label{thm:rig_excess_risk}
\end{theorem}
We prove Theorem \ref{thm:rig_excess_risk} by evaluating the bound in Proposition \ref{pro:excess_risk} at a particular choice of $Q$, and then upper bounding each term by either $\gamma_n(2\eta\alpha, 2\beta\alpha)$ or $\|\fstar\|_{\gH}^2$. The gory details can be found in Appendix \ref{sec:rig_excess_risk}. By Jensen's inequality, $R_n(\vm_{n, \eta, \alpha}) \leq \int R_n(\vg)\mathrm{d}Q_{n, \eta, \alpha}(\vg)$. Since $\vm_{n, \eta,\alpha}$ is identical to the vector of fitted values for kernel ridge regression (with learning rate $2\eta\alpha$), the excess risk bound in Theorem \ref{thm:rig_excess_risk} also applies to kernel ridge regression.

\section{Rates of Convergence for Mercer Kernels}
\label{sec:rates}

We provide worst-case bounds on the relative information gain, which, when combined with Theorem \ref{thm:rig_excess_risk}, can be used to provide explicit rates of convergence for the excess risk of $Q_{n, \eta, \alpha}$.

\subsection{Bounds on the Relative Information Gain}

Aside from one or two tricks, the method with which we bound the relative information gain is identical to the method that \citet{vakili2021information} used to bound the information gain. As a result, we require the same regularity assumptions. A positive-definite kernel $k$ on a set $\mathcal{X}$ is called a Mercer kernel if $\mathcal{X}$ is a compact metric space and the kernel function $k: \mathcal{X} \times \mathcal{X} \to \mathbb{R}$ is continuous. Due to Mercer's theorem (cf. Appendix \ref{sec:mercer}), any Mercer kernel can be expressed as an infinite sum of the form
\begin{equation*}
k(x, x^{\prime}) = \sum_{i=1}^{\infty}\xi\phi_i(x)\phi_i(x^{\prime})\,,
\end{equation*}
where $(\xi_i)_{i=1}^{\infty}$ and $(\phi_i)_{i=1}^{\infty}$ are the (non-negative) eigenvalues and eigenfunctions of the kernel.
\begin{assumption}
Assume that $k$ is a Mercer kernel, and that $\sup_{i \in \sN}\|\phi_i\|_{\infty} \leq \psi$.
\label{ass:mercer}
\end{assumption}
The rate at which the eigenvalues $\xi_1, \xi_2, \dots$ of the kernel decay to zero determines the complexity of the corresponding RKHS. The two eigenvalue decay conditions studied by \citet{vakili2021information} (and many others) are defined as follows.

\begin{assumption}
The polynomial eigenvalue decay condition is satisfied if there exist $C_p > 0$ and $\beta_p > 1$ such that
\begin{equation*}
\xi_i \leq C_pi^{-\beta_p}.
\end{equation*}
The exponential eigenvalue decay condition is satisfied if there exist $C_{e_1}, C_{e_2} > 0$ and $\beta_{e} \in (0, 1]$ such that
\begin{equation*}
\xi_i \leq C_{e_1} \exp(-C_{e_2}i^{\beta_e}).
\end{equation*}
\label{ass:eig}
\end{assumption}
If a kernel is known to satisfy one of these eigenvalue decay conditions, then upper bounds on $\gamma_n(\eta)$ with explicit dependence on $n$ can be given (cf. Corollary 1 in \citealp{vakili2021information}). In particular, under the polynomial eigenvalue decay condition, $\gamma_n(\eta) = \gO((n\eta)^{1/\beta_p}\log^{1-1/\beta_p}(n\eta))$. Under the exponential eigenvalue decay condition, $\gamma_n(\eta) = \gO(\log^{1 + 1/\beta_e}(n\eta))$. A central idea in the method used by \citet{vakili2021information} is to separate the kernel function into $k_{\parallel}(x,x^{\prime}) := \sum_{i=1}^{D}\xi_i\phi_i(x)\phi_i(x^{\prime})$ and  $k_{\perp}(x, x^{\prime}) := \sum_{D+1}^{\infty}\xi_i\phi_i(x)\phi_i(x^{\prime})$. It can be seen that $k_{\parallel}$ is the reproducing kernel of the subspace of $\gH$ spanned by $\phi_1, \dots, \phi_D$ and that  $k_{\parallel}(x,x^{\prime})$ is the reproducing kernel of the subspace of $\gH$ which is orthogonal to $\phi_1, \dots, \phi_D$. Let us write $\mK_n^{\parallel}$ and $\mK_n^{\perp}$ for the corresponding kernel matrices. We define
\begin{equation*}
\delta_D := \sum_{i=D+1}^{\infty}\xi_i\psi^2\,.
\end{equation*}
The first step is to re-write the information gain as a term depending on the rank $D$ kernel matrix $\mK_n^{\parallel}$ and another term depending on both $\mK_n^{\parallel}$ and $\mK_n^{\perp}$. \citet{vakili2021information} show that
\begin{equation}
\gamma_n(\eta) = \frac{1}{2}\log\det(\eta\mK_n^{\parallel} + \mI) + \frac{1}{2}\log\det(\mI + \eta(\mI + \eta\mK_n^{\parallel})^{-1}\mK_n^{\perp})\,.\label{eqn:ig_split}
\end{equation}
From here, \citet{vakili2021information} show that the first term is bounded is by $\frac{1}{2}D\log(1 + \frac{\overline{k}n\eta}{D})$, where $\overline{k} = \sup_{x}|k(x, x)|$, and that the second term is bounded by $\frac{1}{2}n\eta\delta_D$. One then has the upper bound
\begin{equation*}
\gamma_n(\eta) \leq \frac{1}{2}D\log\bigg(1 + \frac{\overline{k}n\eta}{D}\bigg) + \frac{1}{2}n\eta\delta_D\,.
\end{equation*}
The relative information gain satisfies a similar, but slightly tighter bound.
\begin{proposition}
For all $\eta > \beta > 0$ and any integer $D \geq 1$,
\begin{equation*}
\gamma_n(\eta, \beta) \leq \frac{1}{2}D\log\frac{\eta}{\beta} + \frac{1}{2}n\eta\delta_D\,.
\end{equation*}
\label{pro:rig_bound_trick}
\end{proposition}\vspace{-3mm}
The full proof of Proposition \ref{pro:rig_bound_trick} can be found in Appendix \ref{sec:rig_bound_trick}. From \eqref{eqn:ig_split}, it follows that the relative information gain can be re-written as
\begin{equation*}
\gamma_n(\eta, \beta) = \frac{1}{2}\log\frac{\det(\eta\mK_n^{\parallel} + \mI)}{\det(\beta\mK_n^{\parallel} + \mI)} + \frac{1}{2}\log\frac{\det(\mI + \eta(\mI + \eta\mK_n^{\parallel})^{-1}\mK_n^{\perp})}{\det(\mI + \beta(\mI + \beta\mK_n^{\parallel})^{-1}\mK_n^{\perp})}\,.
\end{equation*}
Since $\log\det(\mI + \beta(\mI + \beta\mK_n^{\parallel})^{-1}\mK_n^{\perp}) \geq 0$ (cf. Lemma \ref{lem:non_neg_log_det}), the second term is still upper bounded by $\frac{1}{2}n\eta\delta_D$. In Lemma \ref{lem:gram_det}, we show that the first term can be bounded by $\frac{1}{2}D\log\frac{\eta}{\beta}$, saving a factor of $\log(n)$. Using Proposition \ref{pro:rig_bound_trick}, one can upper bound the relative information gain based on the spectral decay of the kernel via $\delta_D$.
\begin{proposition}
If the polynomial eigenvalue decay condition is satisfied, then
\begin{equation*}
\gamma_n(\eta, \beta) \leq (n\eta C_p\psi^2)^{1/\beta_p}\log^{1-1/\beta_p}(\tfrac{\eta}{\beta}) + \log\tfrac{\eta}{\beta}\,.
\end{equation*}
If the exponential eigenvalue decay condition is satisfied, then
\begin{equation*}
\gamma_n(\eta,\beta) \leq \frac{1}{2}\left(\frac{2}{C_{e,2}}\log(n\eta C_{\beta_e})\right)^{1/\beta_e}\log\frac{\eta}{\beta} + \frac{1}{2}\log\frac{e\eta}{\beta}\,,
\end{equation*}
where $C_{\beta_e} = \frac{C_{e,1}\psi^2}{C_{e,2}}$ if $\beta_e = 1$, and $C_{\beta_e} = \frac{2C_{e,1}\psi^2}{C_{e,2}\beta_e}(\frac{2 - 2\beta_e}{C_{e,2}\beta_e})^{1/\beta_e-1}\exp(\frac{1-\beta_e}{\beta_e})$ if $\beta_e \in (0, 1)$.
\label{pro:rel_info_gain}
\end{proposition}
If $\beta \propto \eta$, then the bound for polynomial decay can be simplified to $\gamma_n(\eta, \beta) = \gO((n\eta)^{1/\beta_e})$. In this case, the bound for exponential decay can be simplified to $\gamma_n(\eta, \beta) = \gO(\log^{1/\beta_e}(n\eta))$. The proof is very similar to that of Corollary 1 in \citet{vakili2021information}, and can be found in Appendix \ref{sec:rig_rates}.

\subsection{Rates of Convergence}

For the case of polynomial eigenvalue decay, by plugging the bound in Proposition \ref{pro:rel_info_gain} into Theorem \ref{thm:rig_excess_risk} and then choosing suitable values of $\eta$ and $\alpha$, we obtain the following rate of convergence.
\begin{corollary}
Suppose that the polynomial eigenvalue decay condition is satisfied, $\eta = \frac{1}{4\sigma^2}$ and $\alpha = n^{-\frac{1}{1 + \beta_p}}$. For any $\delta \in (0, 1]$, with probability at least $1 - \delta$,
\begin{equation*}
\int R_n(\vg)\mathrm{d}Q_{n, \eta, \alpha}(\vg) = \gO\bigg(n^{-\frac{\beta_p}{1 + \beta_p}}\bigg)\,.
\end{equation*}
\label{cor:rate_poly}
\end{corollary}
The proof can be found in Appendix \ref{sec:app_rate_poly}. The rate of convergence in Corollary \ref{cor:rate_poly} is minimax-optimal \citep{dicker2017kernel}. For instance, for kernels satisfying the polynomial decay condition with $\beta_p = 2q$, the unit ball of the corresponding RKHS $\gH$ is a Sobolev space of $q-1$ times absolutely continuous and differentiable functions. In this case, we recover the standard rate of $n^{-\frac{2q}{1 + 2q}}$ for nonparametric regression \citep{tsybakov2009introduction}. By tuning $\alpha$ according to $\sigma$, we can also match the dependence on $\sigma$ in the minimax lower bound for generalised Sobolev ellipsoids in Example 15.23 of \citet{wainwright2019high}. Note, however, that this example considers regression with random design rather than fixed design. For exponential decay, we obtain the following rate of convergence.
\begin{corollary}
Suppose that the exponential eigenvalue decay condition is satisfied, $\eta = \frac{1}{4\sigma^2}$ and $\alpha = 1$. For any $\delta \in (0, 1]$, with probability at least $1 - \delta$,
\begin{equation*}
\int R_n(\vg)\mathrm{d}Q_{n, \eta, \alpha}(\vg) = \gO\bigg(\frac{\log^{1/\beta_e}(n)}{n}\bigg)\,.
\end{equation*}
\label{cor:rate_exp}
\end{corollary}
The proof can be found in Appendix \ref{sec:app_rate_exp}. This rate of convergence for exponential eigenvalue decay is also minimax-optimal \citep{dicker2017kernel}.

\section{Discussion}

We have introduced a new quantity called the relative information gain, which measures the sensitivity of the information gain with respect to the variance of the noise in the responses. We demonstrated that the relative information gain arises naturally from the complexity term of a PAC-Bayesian excess risk bound for Gaussian process regression. Finally, we proved bounds on the relative information gain. When we combined these bounds with the excess risk bound in Theorem \ref{thm:rig_excess_risk}, we recovered minimax-optimal rates of convergence.

In future work, the extent to which the PAC-Bayesian bounds in this paper can be generalised to other regression models could be invesitgated. For instance, one could consider regression with random design, misspecified regression or regression with sparse or other approximate Gaussian process predictors.

\acks{We would like to thank David Reeb for numerous helpful discussions about and comments on an early version of this work. Hamish was funded by the European Research Council (ERC), under the European Union’s Horizon 2020 research and innovation programme (grant agreement 950180).}

\bibliography{main}

\appendix

\section{Auxiliary Lemmas for Section \ref{sec:deff_ig_rig}}

First, we find an expression for the effective dimension in terms of the eigenvalues of the kernel matrix. This expression is well-known (see e.g.\,, Chapter 2.6 in \citealp{williams2006gaussian} or Chapter 5.4 in \citealp{hastie2009elements}).
\begin{lemma}
Let $(\lambda_i)_{i=1}^{n}$ be the eigenvalues of $\mK_n$. For all $n \geq 1$ and $\eta > 0$,
\begin{equation*}
d_n(\eta) = \sum_{i=1}^{n}\frac{\eta\lambda_i}{1 + \eta\lambda_i}\,.
\end{equation*}
\label{lem:deff_eqdef}
\end{lemma}
\begin{proof}
Let $(\vv_i)_{i=1}^{n}$ be the eigenvectors of $\mK_n$. First, we notice that for all $i \in [n]$,
\begin{equation*}
(\mK_n + \tfrac{1}{\eta}\mI)\vv_i = (\lambda_i + 1/\eta)\vv_i\,.
\end{equation*}
It follows that
\begin{equation}
\mK_n(\mK_n + \tfrac{1}{\eta}\mI)^{-1}\vv_i = \tfrac{1}{\lambda_i + 1/\eta}\mK_n\vv_i = \tfrac{\eta\lambda_i}{1 + \eta\lambda_i}\vv_i\,.\label{eqn:eigs}
\end{equation}
Since the trace of a matrix is equal to the sum of its eigenvalues, this concludes the proof.
\end{proof}

\begin{lemma}
Let $(\lambda_i)_{i=1}^{n}$ be the eigenvalues of $\mK_n$. For all $n \geq 1$ and $\eta \geq 0$,
\begin{equation*}
\gamma_n(\eta) = \frac{1}{2}\sum_{i=1}^{n}\log(1 + \eta \lambda_i)\,.
\end{equation*}
\label{lem:ig_eqdef}
\end{lemma}
\begin{proof}
Let $(\vv_i)_{i=1}^{n}$ be the eigenvectors of $\mK_n$. We notice that for all $i \in [n]$,
\begin{equation*}
(\eta\mK_n + \mI)\vv_i = (1 + \eta\lambda_i)\vv_i\,.
\end{equation*}
Since the determinant of a matrix is equal to the product of its eigenvalues,
\begin{equation*}
\gamma_n(\eta) = \frac{1}{2}\log\det(\eta\mK_n + \mI) = \frac{1}{2}\log\prod_{i=1}^{n}(1 + \eta\lambda_i) = \frac{1}{2}\sum_{i=1}^{n}\log(1 + \eta\lambda_i)\,.
\end{equation*}
This concludes the proof.
\end{proof}
An obvious consequence of Lemma \ref{lem:ig_eqdef} is that for all $n \geq 1$ and $\eta > \beta \geq 0$,
\begin{equation*}
\gamma_n(\eta, \beta) = \frac{1}{2}\sum_{i=1}^{n}\log\frac{1 + \eta \lambda_i}{1 + \beta\lambda_i}\,.
\end{equation*}

\begin{lemma}
For any $\lambda \geq 0$, the function $f(\beta) := \frac{1}{\eta - \beta}\log\frac{1 + \eta\lambda}{1 + \beta\lambda}$ is decreasing on $[0, \eta)$.
\label{lem:rig_decr}
\end{lemma}
\begin{proof}
One can verify that the derivative of $f$ is
\begin{equation*}
f^{\prime}(\beta) = \frac{1}{(\eta - \beta)^2}\log\frac{1 + \eta\lambda}{1 + \beta\lambda} - \frac{\lambda}{(\eta - \beta)(1 + \beta\lambda)}\,.
\end{equation*}
Using the inequality $\log(1 + x) \leq x$ for $x \geq 0$, we obtain
\begin{equation*}
\frac{1}{(\eta - \beta)^2}\log\frac{1 + \eta\lambda}{1 + \beta\lambda} = \frac{1}{(\eta - \beta)^2}\log\left(1 + \frac{(\eta - \beta)\lambda}{1 + \beta\lambda}\right) \leq \frac{\lambda}{(\eta - \beta)(1 + \beta\lambda)}\,.
\end{equation*}
It follows that $f^{\prime}(\beta) \leq 0$, and so $f$ must be decreasing on $[0, \eta)$.
\end{proof}

\section{Effective Dimension and Minimum Mean Squared Error}
\label{sec:mmse_deff}

We explain how Proposition \ref{pro:deff_diff} can be derived from Theorem 1 in \citet{guo2005mutual}. We write the model in \eqref{eqn:model} as
\begin{equation*}
\vy_n = \vfstar_n + \veps_n\,,
\end{equation*}
where $\vfstar_n \sim \gN(0, \mK_n)$ and $\veps_n \sim \gN(0, (1/\eta)\mI)$. For an estimate $f(\vy_n)$ of $\vfstar_n$, the mean squared error is $\sE[\|f(\vy_n) - \vfstar_n\|_2^2]$. The minimum achievable mean squared error (w.r.t. the choice of $f$) depends on $\eta$. Let us therefore define the function
\begin{equation*}
\mathrm{MMSE}(\eta) := \min_{f}\sE[\|f(\vy_n) - \vfstar_n\|_2^2]\,.
\end{equation*}
Theorem 1 in \citet{guo2005mutual} states that
\begin{equation*}
\frac{1}{2}\mathrm{MMSE}(\eta) = \gamma_n^{\prime}(\eta)\,.
\end{equation*}
If we can show that $\mathrm{MMSE}(\eta) = d_n(\eta)/\eta$, then this matches the identity in Proposition \ref{pro:deff_diff}. By the least-squares property of the conditional expectation (see e.g.\,, Chapter 9.4 in \citealp{williams1991probability}), the minimum mean squared error w.r.t. $f$ is achieved when
\begin{equation*}
f(\vy_n) = \sE[\vfstar_n|\vy_n] = \mK_n(\mK_n + \tfrac{1}{\eta}\mI)^{-1}\vy_n\,.
\end{equation*}
Thus we can express the minimum mean squared error as
\begin{equation*}
\mathrm{MMSE}(\eta) = \sE[\|\sE[\vfstar_n|\vy_n] - \vfstar_n\|_2^2] = \sE[\|\vfstar_n\|_2^2 - \|\mK_n(\mK_n + \tfrac{1}{\eta}\mI)^{-1}\vy_n\|_2^2]\,.
\end{equation*}
Since $\mK_n(\mK_n + \tfrac{1}{\eta}\mI)^{-1}\vy_n \sim \gN(0, \mK_n(\mK_n + \tfrac{1}{\eta}\mI)^{-1}\mK_n)$, we obtain
\begin{align*}
\sE[\|\vfstar_n\|_2^2 - \|\mK_n(\mK_n + \tfrac{1}{\eta})^{-1}\vy_n\|_2^2] &= \mathrm{tr}(\mK_n) - \mathrm{tr}(\mK_n(\mK_n + \tfrac{1}{\eta})^{-1}\mK_n)\\
&= \mathrm{tr}(\mK_n(\mI - (\mK_n + \tfrac{1}{\eta})^{-1}\mK_n))\\
&= \tfrac{1}{\eta}\mathrm{tr}(\mK_n(\mK_n + \tfrac{1}{\eta}\mI)^{-1})\,.
\end{align*}
We conclude that $\mathrm{MMSE}(\eta) = d_n(\eta)/\eta$.

\section{PAC-Bayesian Bounds}
\label{sec:pac_bayes}

\subsection{Auxiliary Lemmas}
\label{sec:aux_pac_bayes}

We will use the following variational representation of the KL divergence, which was proved by \citet{donsker1976asymptotic}.
\begin{lemma}
For any measurable function $h:\sR^n \to \sR$ and any probability measure $P \in \Delta(\sR^n)$ such that $\int\exp(h(\vg))\mathrm{d}P(\vg) < \infty$,
\begin{equation*}
\sup_{Q \in \Delta(\sR^n)}\left\{\int h(\vg)\mathrm{d}Q(\vg) - \KL(Q||P)\right\} = \log\int\exp(h(\vg))\mathrm{d}P(\vg)\,.
\end{equation*}
If $h$ is bounded above, then the supremum is achieved when
\begin{equation*}
\dd{}{Q}{P}(\vg) \propto \exp(h(\vg))\,.
\end{equation*}
\label{lem:donsker}
\end{lemma}
By rearranging the statement involving the supremum, we see that
\begin{equation*}
\inf_{Q \in \Delta(\sR^n)}\left\{\int h(\vg)\mathrm{d}Q(\vg) + \KL(Q||P)\right\} = -\log\int\exp(-h(\vg))\mathrm{d}P(\vg)\,.
\end{equation*}
If $h$ is bounded below, then the infimum is achieved when
\begin{equation*}
\dd{}{Q}{P}(\vg) \propto \exp(-h(\vg))\,.
\end{equation*}
\begin{lemma}
The function $f(\zeta) := \|(\mK_n(\mK_n + \zeta\mI)^{-1} - \mI)\vy_n\|_2^2$ is increasing on $[0, \infty)$.
\end{lemma}
\label{lem:inc_risk}
\begin{proof}
We notice that
\begin{align*}
\mK_n(\mK_n + \zeta\mI)^{-1} - \mI &= (\mK_n + \zeta\mI)(\mK_n + \zeta\mI)^{-1} - \mI - \zeta(\mK_n + \zeta\mI)^{-1}\\
&= - \zeta(\mK_n + \zeta\mI)^{-1}\,.
\end{align*}
Thus from \eqref{eqn:eigs}, we see that the eigenvalues of $(\mK_n(\mK_n + \zeta\mI)^{-1} - \mI)^2$ are $(\zeta^2/(\zeta + \lambda_i)^2)_{i=1}^{n}$. Fix $\zeta_2 \geq \zeta_1 \geq 0$. For each $i \in [n]$,
\begin{equation*}
\zeta_1^2(\zeta_2 + \lambda_i)^2 = \zeta_1^2\zeta_2^2 + 2\zeta_1^2\zeta_2\lambda_i + \zeta_1^2\lambda_i^2 \leq \zeta_1^2\zeta_2^2 + 2\zeta_1\zeta_2^2\lambda_i + \zeta_2^2\lambda_i^2 = \zeta_2^2(\zeta_1 + \lambda_i)^2\,.
\end{equation*}
It follows that each of the functions $g_i(\zeta) := \zeta^2/(\zeta + \lambda_i)^2$ is increasing on $[0, \infty)$. Hence,
\begin{equation*}
(\mK_n(\mK_n + \zeta_1\mI)^{-1} - \mI)^2 \preccurlyeq (\mK_n(\mK_n + \zeta_2\mI)^{-1} - \mI)^2\,.
\end{equation*}
This concludes the proof.
\end{proof}

\begin{lemma}
For all $n \geq 1$, $\eta > 0$ and $\alpha \geq 0$,
\begin{equation*}
\mK_{n,\eta,\alpha} = \alpha\mK_n(2\eta\alpha\mK_n + \mI)^{-1}\,.
\end{equation*}
\label{lem:cov_as_deff}
\end{lemma}
\begin{proof}
By adding zero to $\mK_{n,\eta,\alpha}$, we obtain
\begin{align*}
\mK_{n,\eta,\alpha} &= \alpha\mK_n - \alpha\mK_n(\mK_n + \tfrac{1}{2\eta\alpha}\mI)^{-1}(\mK_n + \tfrac{1}{2\eta\alpha}\mI - \tfrac{1}{2\eta\alpha}\mI)\\
&= \tfrac{1}{2\eta}\mK_n(\mK_n + \frac{1}{2\eta\alpha}\mI)^{-1}\\
&= \alpha\mK_n(2\eta\alpha \mK_n + \mI)^{-1}\,.
\end{align*}
This concludes the proof.
\end{proof}
As a result of Lemma \ref{lem:cov_as_deff}, $\mathrm{tr}(\mK_{n,\eta,\alpha}) = \frac{1}{2\eta}d_n(2\eta\alpha)$.

\begin{corollary}
For all $n \geq 1$, $\eta > 0$ and $\alpha \geq 0$,
\begin{equation}
\alpha\mK_{n,\eta,\alpha}^{-1}\mK_n = 2\eta\alpha\mK_n + \mI\,.
\end{equation}
\label{cor:inv_bollocks}
\end{corollary}
\begin{proof}
From Lemma \ref{lem:cov_as_deff}, we have the identity
\begin{equation*}
\alpha\mK_{n,\eta,\alpha}^{-1}\mK_n(2\eta\alpha\mK_n + \mI)^{-1} = \mI\,.
\end{equation*}
Post-multiplying both sides by $2\eta\alpha\mK_n + \mI$ gives the desired result.
\end{proof}

\begin{corollary}
For all $n \geq 1$, $\eta > \beta > 0$ and $\alpha \geq 0$,
\begin{equation*}
\frac{1}{2}\log\frac{\det(\mK_{n,\beta,\alpha})}{\det(\mK_{n,\eta,\alpha})} = \gamma_n(2\eta\alpha, 2\beta\alpha)\,.
\end{equation*}
\label{cor:logdet}
\end{corollary}
\begin{proof}
Using Lemma \ref{lem:cov_as_deff} and standard properties of determinants,
\begin{equation*}
\frac{1}{2}\log\frac{\det(\mK_{n,\beta,\alpha})}{\det(\mK_{n,\eta,\alpha})} = \frac{1}{2}\log\frac{\det((2\beta\alpha\mK_n + \mI)^{-1})}{\det((2\eta\alpha\mK_n + \mI)^{-1})} = \frac{1}{2}\log\frac{\det(2\eta\alpha\mK_n + \mI)}{\det(2\beta\alpha\mK_n + \mI)}\,.
\end{equation*}
This concludes the proof.
\end{proof}

\begin{corollary}
For all $n \geq 1$, $\eta > \beta > 0$ and $\alpha > 0$,
\begin{equation*}
\mathrm{tr}(\mK_{n,\eta,\alpha}) \leq \frac{1}{\eta - \beta}\gamma_n(2\eta\alpha, 2\beta\alpha)\,.
\end{equation*}
\label{cor:trace_bollocks}
\end{corollary}
\begin{proof}
From Lemma \ref{lem:cov_as_deff} and Proposition \ref{pro:interp}, we obtain
\begin{equation*}
\mathrm{tr}(\mK_{n,\eta,\alpha}) = \frac{1}{2\eta}d_n(2\eta\alpha) \leq \frac{1}{2\eta}\frac{4\eta\alpha}{2\eta\alpha - 2\beta\alpha}\gamma_n(2\eta\alpha, 2\beta\alpha) = \frac{1}{\eta-\beta}\gamma_n(2\eta\alpha, 2\beta\alpha)\,.
\end{equation*}
This concludes the proof.
\end{proof}

\begin{lemma}
For all $n \geq 1$, $\eta > \beta > 0$ and $\alpha > 0$,
\begin{equation*}
\mathrm{tr}(\mK_{n,\beta,\alpha}^{-1}\mK_{n,\eta,\alpha}) \leq n\,.
\end{equation*}
\label{lem:trace_bollocks}
\end{lemma}
\begin{proof}
Let $(\lambda_i)_{i=1}^{n}$ and $(\vv_i)_{i=1}^{n}$ be the eigenvalues and eigenvectors of $\mK_n$. From Lemma \ref{lem:cov_as_deff} and \eqref{eqn:eigs}, it follows that, for all $i \in [n]$,
\begin{equation*}
\mK_{n,\eta,\alpha}\vv_i = \frac{\alpha\lambda_i}{2\eta\alpha\lambda_i + 1}\vv_i\,.
\end{equation*}
As long as $\mK_n$ is positive-definite, this implies that, for all $i \in [n]$,
\begin{equation*}
\mK_{n,\beta,\alpha}^{-1}\vv_i = \frac{2\beta\alpha\lambda_i + 1}{\alpha\lambda_i}\vv_i\,.
\end{equation*}
In particular, the eigenvalues of $\mK_{n,\beta,\alpha}^{-1}\mK_{n,\eta,\alpha}$ are $(\frac{2\beta\alpha\lambda_i + 1}{2\eta\alpha\lambda_i + 1})_{i=1}^{n}$. Therefore,
\begin{equation*}
\mathrm{tr}(\mK_{n,\beta,\alpha}^{-1}\mK_{n,\eta,\alpha}) = \sum_{i=1}^{n}\frac{2\beta\alpha\lambda_i + 1}{2\eta\alpha\lambda_i + 1} \leq \sum_{i=1}^{n}\frac{2\eta\alpha\lambda_i + 1}{2\eta\alpha\lambda_i + 1} = n\,.
\end{equation*}
This concludes the proof.
\end{proof}

\subsection{A PAC-Bayesian Bound for Sums of Independent Random Variables}
\label{sec:pb_indep}

Proposition \ref{pro:excess_risk} is a special case of a localised PAC-Bayesian bound for collections of sums of independent random variables. We consider a collection of random variables $(Z_i(\vg))_{i \in [n], \vg \in \sR^n}$, such that for each $\vg \in \sR^n$, $(Z_i(\vg))_{i=1}^{n}$ is a sequence of independent random variables. We let $\mu_i(\vg) := \sE[Z_i(\vg)]$ and $\psi_i(\vg, \eta) := \log\sE[\exp(\eta (Z_i(\vg)-\mu_i(\vg))]$ denote the mean and the cumulant generating function of $Z_i(\vg)$. We assume that for all $\eta \in \sR$, $i \in [n]$ and $\vg \in \sR^n$, $\sE[\exp(\eta (Z_i(\vg) - \mu_i(\vg))] < \infty$. For an arbitrary distribution $P \in \Delta(\sR^n)$, we define the localised prior $Q_{\mu,\beta}$ by
\begin{equation*}
\dd{}{Q_{\mu,\beta}}{P}(\vg) \propto \exp\left(\sum_{i=1}^{n}\beta\mu_i(\vg) - \psi_i(\vf_n, -\beta)\right)\,.
\end{equation*}
Note that while $Q_{\mu,\beta}$ depends the unobserved quantities $\mu_i(\vg)$ and $\psi_i(\vg, -\beta)$, it does not depend on the random draw of $(Z_i(\vg))_{i \in [n], \vg \in \sR^n}$. We define the empirical approximation $Q_{Z,\beta}$ of the localised prior by
\begin{equation*}
\dd{}{Q_{Z,\beta}}{P}(\vg) \propto \exp\left(\sum_{i=1}^{n}\beta Z_i(\vg)\right)
\end{equation*}
In contrast to $Q_{\mu,\beta}$, $Q_{Z,\beta}$ does depend on the random draw of $(Z_i(\vg))_{i \in [n], \vg \in \sR^n}$, but does not depend on any unobservable quantities. The following proposition combines a well-known PAC-Bayesian bound for sums of independent random variables (cf. Proposition 2.1 in \citealp{catoni2017dimension}) with Catoni's localisation technique (cf. Section 1.3.4 in \citealp{catoni2007pac}).

\begin{proposition}
For any $n \geq 1$, any $\delta \in (0, 1]$, any $\eta > 0$, any $\beta \in [0, \eta)$ and any $P \in \Delta(\sR^n)$, with probability at least $1 - \delta$, $\forall Q \in \Delta(\sR^n)$,
\begin{equation*}
\int\left[\sum_{i=1}^{n}(\eta - \beta)(Z_i(\vg) - \mu_i(\vg)) - \psi_i(\vg, \eta) - \psi_i(\vg, -\beta)\right]\mathrm{d}Q(\vg) - \KL(Q||Q_{\beta Z_n}) \leq 2\log\tfrac{1}{\delta}\,.
\end{equation*}
\label{pro:pb_indep}
\end{proposition}
The proof is a fairly straightforward combination of the proofs of Proposition 2.1 from \citet{catoni2017dimension} and some of the derivations in Section 1.3.4 in \citet{catoni2007pac}.
\begin{proof}
Fix an arbitrary $Q \in \Delta(\sR^n)$. We begin by finding a relationship between $\KL(Q||Q_{\mu,\beta})$ and $\KL(Q||Q_{Z,\beta})$. From the definitions of $Q_{\mu,\beta}$ and $Q_{Z,\beta}$, we obtain
\begin{align}
\KL(Q||Q_{\mu,\beta}) &= \int\log\dd{}{Q}{Q_{Z,\beta}}(\vg)\dd{}{Q_{Z,\beta}}{Q_{\mu,\beta}}(\vg)\mathrm{d}Q(\vg)\label{eqn:local_kl}\\
&= \KL(Q||Q_{Z,\beta})\nonumber\\
&+ \int\log\frac{\exp(\sum_{i=1}^{n}\beta Z_i(\vg))\int\exp(\sum_{i=1}^{n}\beta\mu_i(\vg) - \psi_i(\vg,-\beta))\mathrm{d}P(\vg)}{\exp(\sum_{i=1}^{n}\beta\mu_i(\vg) - \psi_i(\vg,-\beta))\int\exp(\sum_{i=1}^{n}\beta Z_i(\vg))\mathrm{d}P(\vg)}\mathrm{d}Q(\vg)\nonumber\\
&= \KL(Q||Q_{Z,\beta}) + \int\left[\sum_{i=1}^{n}\beta Z_i(\vg) - \beta \mu_i(\vg) + \psi_i(\vg,-\beta)\right]\mathrm{d}Q(\vg)\nonumber\\
&+ \log\frac{\int\exp\left(\sum_{i=1}^{n}\beta\mu_i(\vg) - \psi_i(\vg,-\beta)\right)\mathrm{d}P(\vg)}{\int\exp\left(\sum_{i=1}^{n}\beta Z_i(\vg)\right)\mathrm{d}P(\vg)}\,.\nonumber
\end{align}
Using Lemma \ref{lem:donsker}, Tonelli's theorem and independence, for any $\eta > 0$, we obtain
\begin{align}
\sE&\left[\exp\left(\sup_{Q \in \Delta(\sR^n)}\left\{\int\sum_{i=1}^{n}\big[\eta(Z_i(\vg) - \mu_i(\vg)) - \psi_i(\vg,\eta)\big]\mathrm{d}Q(\vg) - \KL(Q||Q_{\mu,\beta})\right\}\right)\right]\label{eqn:pbi_conc}\\
&= \sE\left[\int\exp\left(\sum_{i=1}^{n}\eta(Z_i(\vg) - \mu_i(\vg)) - \psi_i(\vg,\eta)\right)\mathrm{d}Q_{\mu,\beta}(\vg)\right]\nonumber\\
&= \int\sE\left[\exp\left(\sum_{i=1}^{n}\eta(Z_i(\vg) - \mu_i(\vg)) - \psi_i(\vg,\eta)\right)\right]\mathrm{d}Q_{\mu,\beta}(\vg)\nonumber\\
&= \int\prod_{i=1}^{n}\sE\left[\exp\left(\eta(Z_i(\vg) - \mu_i(\vg)) - \psi_i(\vg,\eta)\right)\right]\mathrm{d}Q_{\mu,\beta}(\vg)\nonumber\\
&= \int\prod_{i=1}^{n}\frac{\exp(\psi_i(\vg,\eta))}{\exp(\psi_i(\vg,\eta))}\mathrm{d}Q_{\mu,\beta}(\vg) = 1\,.\nonumber
\end{align}
Next, using Lemma \ref{lem:donsker} and Jensen's inequality, we obtain
\begin{align}
\sE&\left[\frac{\int\exp\left(\sum_{i=1}^{n}\beta\mu_i(\vg) - \psi_i(\vg,-\beta)\right)\mathrm{d}P(\vg)}{\int\exp\left(\sum_{i=1}^{n}\beta Z_i(\vg)\right)\mathrm{d}P(\vg)}\right]\label{eqn:pbz_conc}\\
&= \sE\left[\frac{\exp(\sup_{Q \in \Delta(\sR^d)}\{\int\sum_{i=1}^{n}\beta\mu_i(\vg) - \psi_i(\vg,-\beta)\mathrm{d}Q(\vg) - \KL(Q||P)\})}{\exp(\sup_{Q \in \Delta(\sR^d)}\{\int\sum_{i=1}^{n}\beta Z_i(\vg)\mathrm{d}Q(\vg) - \KL(Q||P)\})}\right]\nonumber\\
&\leq \sE\left[\exp\left(\int\sum_{i=1}^{n}-\beta(Z_i(\vg) - \mu_i(\vg)) - \psi_i(\vg, -\beta)\mathrm{d}Q_{\mu,\beta}(\vg)\right)\right]\nonumber\\
&\leq \sE\left[\int\exp\left(\sum_{i=1}^{n}-\beta(Z_i(\vg) - \mu_i(\vg)) - \psi_i(\vg, -\beta)\right)\mathrm{d}Q_{\mu,\beta}(\vg)\right] = 1\,.\nonumber
\end{align}
The last step follows from \eqref{eqn:pbi_conc} with $\eta = -\beta$. Using \eqref{eqn:local_kl}, then the Cathy-Schwarz inequality, and then \eqref{eqn:pbi_conc} and \eqref{eqn:pbz_conc}, we obtain
\begin{align*}
\sE&\bigg[\exp\bigg(\frac{1}{2}\sup_{Q \in \Delta(\sR^d)}\bigg\{\int\sum_{i=1}^{n}(\eta-\beta)(Z_i(\vg) - \mu_i(\vg)) - \psi_i(\vg, \eta) - \psi_i(\vg, -\beta)\mathrm{d}Q(\vg) - \KL(Q||Q_{Z,\beta})\bigg\}\bigg)\bigg]\\
&= \sE\bigg[\exp\bigg(\frac{1}{2}\sup_{Q \in \Delta(\sR^d)}\bigg\{\int\sum_{i=1}^{n}\eta(Z_i(\vg) - \mu_i(\vg)) - \psi_i(\vg, \eta)\mathrm{d}Q(\vg) - \KL(Q||Q_{\mu,\beta})\bigg\}\bigg)\\
&\qquad~\times\exp\bigg(\frac{1}{2}\log\frac{\int\exp\left(\sum_{i=1}^{n}\beta\mu_i(\vg) - \psi_i(\vg,-\beta)\right)\mathrm{d}P(\vg)}{\int\exp\left(\sum_{i=1}^{n}\beta Z_i(\vg)\right)\mathrm{d}P(\vg)}\bigg)\bigg]\\
&\leq \sE\bigg[\exp\bigg(\sup_{Q \in \Delta(\sR^d)}\bigg\{\int\sum_{i=1}^{n}\eta(Z_i(\vg) - \mu_i(\vg)) - \psi_i(\vg, \eta)\mathrm{d}Q(\vg) - \KL(Q||Q_{\mu,\beta})\bigg\}\bigg)\bigg]^{1/2}\\
&\times \sE\bigg[\frac{\int\exp\left(\sum_{i=1}^{n}\beta\mu_i(\vg) - \psi_i(\vg,-\beta)\right)\mathrm{d}P(\vg)}{\int\exp\left(\sum_{i=1}^{n}\beta Z_i(\vg)\right)\mathrm{d}P(\vg)}\bigg]^{1/2}\\
&\leq 1\,.
\end{align*}
The statement now follows from Markov's inequality.
\end{proof}
\noindent Notice that if every occurence of $\psi_i(\vg, \eta)$ (and $\psi_i(\vg, -\beta)$) is replaced with an upper bound on $\psi_i(\vg, \eta)$ (or $\psi_i(\vg, -\beta)$), then the proof still goes through. This includes the occurences of $\psi_i(\vg, -\beta)$ in the localised prior $Q_{\mu,\beta}$.

\subsection{Proof of Proposition \ref{pro:excess_risk}}
\label{sec:pb_excess_risk}
\begin{proofof}{Proposition \ref{pro:excess_risk}}
We set
\begin{equation*}
Z_i(\vg) = (\fstar(x_i)-y_i)^2 - (g_i - y_i)^2 = -(g_i - \fstar(x_i))^2 + 2(g_i - \fstar(x_i))\eps_i\,.
\end{equation*}
With this choice of $Z_i(\vg)$, $\mu_i(\vg) = -(g_i - \fstar(x_i))^2$, and for all $\eta \in \sR$,
\begin{equation}
\psi_i(\vg, \eta) = \log\sE[\exp(2\eta(g_i - \fstar(x_i))\eps_i)] \leq 2\sigma^2\eta^2(g_i - \fstar(x_i))^2\,.\label{eqn:cgf_bound}
\end{equation}
In addition, we choose $P = P_{\alpha}$. From the definition of $Z_i(\vg)$, $Q_{Z,\beta}$ is given by
\begin{equation*}
\dd{}{Q_{Z,\beta}}{P_{\alpha}}(\vg) \propto \exp(n\beta(r_n(\vfstar_n) - r_n(\vg))) \propto \exp(-n\beta r_n(\vg))\,.
\end{equation*}
Namely, $Q_{Z,\beta} = Q_{n, \beta, \alpha}$. Even though we will not need it, we can also determine the expression for the localised prior $Q_{\mu,\beta}$ (or rather, for $\mathrm{d}Q_{\mu,\beta}/\mathrm{d}P$). From the bound in \eqref{eqn:cgf_bound}, it follows that $Q_{\mu,\beta}$ is given by
\begin{equation*}
\dd{}{Q_{\mu,\beta}}{P_{\alpha}}(\vg) \propto \exp(-n(\beta + 2\sigma^2\beta^2)R_n(\vg))\,.
\end{equation*}
One can notice that $Q_{\mu,\beta}$ assigns higher density to vectors $\vg$ for which the excess risk is small, whereas $Q_{Z,\beta}$ assigns higher density to vectors $\vg$ for which the empirical risk is small. Next, we substitute the definitions of/bounds on $Z_i(\vg)$, $\mu_i(\vg)$ and $\psi_i(\vg, \eta)$ into Proposition \ref{pro:pb_indep}. In particular, for any $n \geq 1$, any $\delta \in (0,1]$, any $\eta > \beta \geq 0$ and any $P \in \Delta(\sR^n)$, with probability at least $1 - \delta$, for all $Q \in \Delta(\sR^n)$,
\begin{equation*}
\int n(\eta - \beta)(r_n(\vfstar_n) - r_n(\vg)) + n(\eta - 2\sigma^2\eta^2 - \beta - 2\sigma^2\beta^2)R_n(\vg)\mathrm{d}Q(\vg) - \KL(Q||Q_{n,\beta,\alpha}) \leq 2\log\tfrac{1}{\delta}\,.
\end{equation*}
If $\eta$ and $\beta$ are chosen such that $\eta - 2\sigma^2\eta^2 - \beta - 2\sigma^2\beta^2 > 0$, then we can arrange this inequality to obtain
\begin{equation}
\int R_n(\vg)\mathrm{d}Q(\vg) \leq \frac{\int(\eta-\beta)(r_n(\vg) - r_n(\vfstar_n))\mathrm{d}Q(\vg)}{\eta - 2\sigma^2\eta^2 - \beta - 2\sigma^2\beta^2} + \frac{\KL(Q||Q_{n,\beta,\alpha}) + 2\log\tfrac{1}{\delta}}{n(\eta - 2\sigma^2\eta^2 - \beta - 2\sigma^2\beta^2)}\,.\label{eqn:local_pb}
\end{equation}
All that remains is to show that $Q_{n,\eta,\alpha}$ minimises the right-hand side of this inequality w.r.t. $Q$. Since $r_n$ is bounded below (by 0), Lemma \ref{lem:donsker} tells us that the infimum of the right-hand side is achieved when
\begin{equation*}
\dd{}{Q}{Q_{n,\beta,\alpha}}(\vg) \propto \exp(-n(\eta-\beta)r_n(\vg))\,.
\end{equation*}
For $Q = Q_{n,\eta, \alpha}$,
\begin{align*}
\dd{}{Q_{n,\eta,\alpha}}{Q_{n,\beta,\alpha}}(\vg) &= \dd{}{Q_{n,\eta,\alpha}}{P_{\alpha}}(\vg)\dd{}{P_{\alpha}}{Q_{n,\beta,\alpha}}(\vg)\\
&\propto \exp(-n\eta r_n(\vg))\exp(n\beta r_n(\vg))\\
&\propto \exp(-n(\eta - \beta) r_n(\vg))\,.
\end{align*}
Therefore $Q_{n, \eta, \alpha}$ is indeed a minimiser of the right-hand side of \eqref{eqn:local_pb}. Since \eqref{eqn:local_pb} holds simultaneously for all $Q \in \Delta(\sR^n)$, under the same conditions as before, with probability at least $1 - \delta$,
\begin{equation*}
\int R_n(\vg)\mathrm{d}Q_{n,\eta,\alpha}(\vg) \leq \inf_{Q}\left\{\frac{\int(\eta-\beta)(r_n(\vg) - r_n(\vfstar_n))\mathrm{d}Q(\vg)}{\eta - 2\sigma^2\eta^2 - \beta - 2\sigma^2\beta^2} + \frac{\KL(Q||Q_{n,\beta,\alpha}) + 2\log\tfrac{1}{\delta}}{n(\eta - 2\sigma^2\eta^2 - \beta - 2\sigma^2\beta^2)}\right\}\,.
\end{equation*}
This concludes the proof.
\end{proofof}

\subsection{Proof of Theorem \ref{thm:rig_excess_risk}}
\label{sec:rig_excess_risk}

\begin{proofof}{Theorem \ref{thm:rig_excess_risk}}
We define the constrained least squares estimate
\begin{equation*}
\wh f := \argmin_{f \in \gH, \|f\|_{\gH} \leq \|\fstar\|_{\gH}}\bigg\{\sum_{i=1}^{n}(f(x_i) - y_i)^2\bigg\}\,.
\end{equation*}
For some $\zeta \geq 0$, one can express $\wh f$ via
\begin{equation*}
\wh f(x) = \vk_n^{\top}(x)(\mK_n + \zeta\mI)^{-1}\vy_n\,.
\end{equation*}
In addition, the vector of fitted values $\wh\vf_n := [\wh f(x_1), \dots, \wh f(x_n)]^{\top}$ is
\begin{equation*}
\wh\vf_n = \mK_n(\mK_n + \zeta\mI)^{-1}\vy_n\,.
\end{equation*}
From the definition of $\wh f$, we have $r_n(\wh\vf_n) \leq r_n(\vfstar_n)$ and
\begin{equation}
\vy_n^{\top}(\mK_n + \zeta\mI)^{-1}\mK_n(\mK_n + \zeta\mI)^{-1}\vy_n = \|\wh f\|_{\gH}^2 \leq \|\fstar\|_{\gH}^2\,.\label{eqn:fstar_norm}
\end{equation}
We want to find an upper bound for
\begin{equation*}
\inf_{Q \in \Delta(\sR^n)}\bigg\{\frac{\int(\eta-\beta)(r_n(\vg) - r_n(\vfstar_n))\mathrm{d}Q(\vg)}{\eta - 2\sigma^2\eta^2 - \beta - 2\sigma^2\beta^2} + \frac{\KL(Q||Q_{n,\beta,\alpha}) + 2\log\tfrac{1}{\delta}}{n(\eta - 2\sigma^2\eta^2 - \beta - 2\sigma^2\beta^2)}\bigg\}\,.
\end{equation*}
We consider two cases. First, suppose that $\frac{1}{2\beta\alpha} \leq \zeta$. We set $\wh Q = \gN(\vm_{n,\beta,\alpha}, \mK_{n, \eta, \alpha})$. Since the empirical risk $r_n(\vf_n^{(\zeta)})$ of the $\vf_n^{(\zeta)} := \mK_n(\mK_n + \zeta\mI)^{-1}\vy_n$ is increasing in $\zeta$ (cf. Lemma \ref{lem:inc_risk}), and since $\vm_{n,\beta,\alpha} = \vf_n^{(1/(2\beta\alpha))}$, it follows that
\begin{equation*}
r_n(\vm_{n,\beta,\alpha}) \leq r_n(\wh\vf_n) \leq r_n(\vfstar_n)\,.
\end{equation*}
Using this inequality, along with Corollary \ref{cor:trace_bollocks}, we obtain
\begin{align*}
\int(\eta-\beta)(r_n(\vg) - r_n(\vfstar_n))\mathrm{d}\wh Q(\vg) &= (\eta-\beta)(r_n(\vm_{n,\beta,\alpha}) - r_n(\vfstar_n)) + \frac{\eta-\beta}{n}\mathrm{tr}(\mK_{n,\eta,\alpha})\\
&\leq \frac{1}{n}\gamma_n(2\eta\alpha, 2\beta\alpha)\,.
\end{align*}
Using Lemma \ref{lem:trace_bollocks}, Corollary \ref{cor:logdet} and the expression for the KL divergence between two Gaussians (with the same mean), we see that
\begin{equation*}
\KL(\wh Q||Q_{n,\beta,\alpha}) = \frac{1}{2}\bigg(\mathrm{tr}(\mK_{n,\beta,\alpha}^{-1}\mK_{n,\beta,\alpha}) - n + \log\frac{\det(\mK_{n,\beta,\alpha})}{\det(\mK_{n,\eta,\alpha})}\bigg) \leq \gamma_n(2\eta\alpha, 2\beta\alpha)\,.
\end{equation*}
We have shown that, for the case where $\frac{1}{2\beta\alpha} \leq \zeta$,
\begin{align*}
\inf_{Q \in \Delta(\sR^n)}&\bigg\{\frac{\int(\eta-\beta)(r_n(\vg) - r_n(\vfstar_n))\mathrm{d}Q(\vg)}{\eta - 2\sigma^2\eta^2 - \beta - 2\sigma^2\beta^2} + \frac{\KL(Q||Q_{n,\beta,\alpha}) + 2\log\tfrac{1}{\delta}}{n(\eta - 2\sigma^2\eta^2 - \beta - 2\sigma^2\beta^2)}\bigg\}\\
&\leq \frac{\int(\eta-\beta)(r_n(\vg) - r_n(\vfstar_n))\mathrm{d}\wh Q(\vg)}{\eta - 2\sigma^2\eta^2 - \beta - 2\sigma^2\beta^2} + \frac{\KL(\wh Q||Q_{n,\beta,\alpha}) + 2\log\tfrac{1}{\delta}}{n(\eta - 2\sigma^2\eta^2 - \beta - 2\sigma^2\beta^2)}\\
&\leq \frac{2\gamma_n(2\eta\alpha, 2\beta\alpha) + 2\log\tfrac{1}{\delta}}{n(\eta - 2\sigma^2\eta^2 - \beta - 2\sigma^2\beta^2)}\,.
\end{align*}
Next, we consider the second case. Suppose that $\frac{1}{2\alpha\beta} > \zeta$. We set $\wh Q = \gN(\wh\vf_n, \mK_{n,\eta,\alpha})$. Since $r_n(\wh\vf_n) \leq r_n(\vfstar_n)$, Corollary \ref{cor:trace_bollocks} tells us that
\begin{align*}
\int(\eta-\beta)(r_n(\vg) - r_n(\vfstar_n))\mathrm{d}\wh Q(\vg) &= (\eta-\beta)(r_n(\wh\vf_n) - r_n(\vfstar_n)) + \frac{\eta-\beta}{n}\mathrm{tr}(\mK_{n,\eta,\alpha})\\
&\leq \frac{1}{n}\gamma_n(2\eta\alpha, 2\beta\alpha)\,.
\end{align*}
With this choice of $\wh Q$, the KL divergence $\KL(\wh Q||Q_{n,\beta,\alpha})$ is equal to
\begin{equation*}
\frac{1}{2}\bigg(\mathrm{tr}(\mK_{n,\beta,\alpha}^{-1}\mK_{n,\beta,\alpha}) - n + (\wh\vf_n - \vm_{n,\beta,\alpha})^{\top}\mK_{n,\beta,\alpha}^{-1}(\wh\vf_n - \vm_{n,\beta,\alpha}) + \log\frac{\det(\mK_{n,\beta,\alpha})}{\det(\mK_{n,\eta,\alpha})}\bigg)\,.
\end{equation*}
As previously, the trace (minus $n$) and log-determinant terms can be bounded using Lemma \ref{lem:trace_bollocks} and Corollary \ref{cor:logdet}. After some slightly unpleasant calculation, one can show that the remaining quadratic term is upper bounded by the squared RKHS norm of $\fstar$. First, we expand the quadratic term, and obtain
\begin{equation*}
(\wh\vf_n - \vm_{n,\beta,\alpha})^{\top}\mK_{n,\beta,\alpha}^{-1}(\wh\vf_n - \vm_{n,\beta,\alpha}) = \wh\vf_n^{\top}\mK_{n,\beta,\alpha}^{-1}\wh\vf_n - 2\wh\vf_n^{\top}\mK_{n,\beta,\alpha}^{-1}\vm_{n,\beta,\alpha} + \vm_{n,\beta,\alpha}^{\top}\mK_{n,\beta,\alpha}^{-1}\vm_{n,\beta,\alpha}\,.
\end{equation*}
Using Corollary \ref{cor:inv_bollocks} and \eqref{eqn:fstar_norm}, we see that
\begin{align*}
\wh\vf_n^{\top}\mK_{n,\beta,\alpha}^{-1}\wh\vf_n &= \vy_n^{\top}(\mK_n + \zeta\mI)^{-1}\mK_n\mK_{n,\beta,\alpha}^{-1}\mK_n(\mK_n + \zeta\mI)^{-1}\vy_n\\
&= 2\beta\vy_n^{\top}(\mK_n + \zeta\mI)^{-1}\mK_n(\mK_n + \tfrac{1}{2\beta\alpha}\mI)(\mK_n + \zeta\mI)^{-1}\vy_n\\
&= 2\beta\vy_n^{\top}(\mK_n + \zeta\mI)^{-1}\mK_n(\mK_n + \zeta\mI + (\tfrac{1}{2\beta\alpha} - \zeta)\mI)(\mK_n + \zeta\mI)^{-1}\vy_n\\
&= 2\beta\vy_n^{\top}(\mK_n + \zeta\mI)^{-1}\mK_n\vy_n + 2\beta(\tfrac{1}{2\beta\alpha} - \zeta)\vy_n^{\top}(\mK_n + \zeta\mI)^{-1}\mK_n(\mK_n + \zeta\mI)^{-1}\vy_n\\
&\leq 2\beta\vy_n^{\top}(\mK_n + \zeta\mI)^{-1}\mK_n\vy_n + \tfrac{1}{\alpha}\|\fstar\|_{\gH}^2\,.
\end{align*}
The second term in the quadratic expansion can be re-written using Corollary \ref{cor:inv_bollocks} again. In particular,
\begin{align*}
2\wh\vf_n^{\top}\mK_{n,\beta,\alpha}^{-1}\vm_{n,\beta,\alpha} &= 2\vy_n^{\top}(\mK_n + \zeta\mI)^{-1}\mK_n\mK_{n,\beta,\alpha}^{-1}\mK_n(\mK_n + \tfrac{1}{2\beta\alpha}\mI)^{-1}\vy_n\\
&= 4\beta\vy_n^{\top}(\mK_n + \zeta\mI)^{-1}\mK_n(\mK_n + \tfrac{1}{2\beta\alpha}\mI)(\mK_n + \tfrac{1}{2\beta\alpha}\mI)^{-1}\vy_n\\
&= 4\beta\vy_n^{\top}(\mK_n + \zeta\mI)^{-1}\mK_n\vy_n\,.
\end{align*}
The third term in the quadratic expansion can be re-written in the same way. In particular,
\begin{align*}
\vm_{n,\beta,\alpha}^{\top}\mK_{n,\beta,\alpha}^{-1}\vm_{n,\beta,\alpha} &= \vy_n^{\top}(\mK_n + \tfrac{1}{2\beta\alpha}\mI)^{-1}\mK_n\mK_{n,\beta,\alpha}^{-1}\mK_n(\mK_n + \tfrac{1}{2\beta\alpha}\mI)^{-1}\vy_n\\
&= 2\beta\vy_n^{\top}(\mK_n + \tfrac{1}{2\beta\alpha}\mI)^{-1}\mK_n(\mK_n + \tfrac{1}{2\beta\alpha}\mI)(\mK_n + \tfrac{1}{2\beta\alpha}\mI)^{-1}\vy_n\\
&= 2\beta\vy_n^{\top}(\mK_n + \tfrac{1}{2\beta\alpha}\mI)^{-1}\mK_n\vy_n\,.
\end{align*}
We have now shown that
\begin{align*}
(\wh\vf_n - \vm_{n,\beta,\alpha})^{\top}\mK_{n,\beta,\alpha}^{-1}(\wh\vf_n - \vm_{n,\beta,\alpha}) &\leq \tfrac{1}{\alpha}\|\fstar\|_{\gH}^2 + 2\beta\vy_n^{\top}(\mK_n + \tfrac{1}{2\beta\alpha}\mI)^{-1}\mK_n\vy_n\\
&- 2\beta\vy_n^{\top}(\mK_n + \zeta\mI)^{-1}\mK_n\vy_n\,.
\end{align*}
Since $\frac{1}{2\beta\alpha} > \zeta$ and the eigenvalues of $(\mK_n + \zeta\mI)^{-1}\mK_n$ are decreasing in $\zeta$, it follows that $(\mK_n + \frac{1}{2\beta\alpha}\mI)^{-1}\mK_n \preccurlyeq (\mK_n + \zeta\mI)^{-1}\mK_n$. Therefore,
\begin{equation*}
(\wh\vf_n - \vm_{n,\beta,\alpha})^{\top}\mK_{n,\beta,\alpha}^{-1}(\wh\vf_n - \vm_{n,\beta,\alpha}) \leq \tfrac{1}{\alpha}\|\fstar\|_{\gH}^2\,.
\end{equation*}
For the second case where $\frac{1}{2\beta\alpha} > \zeta$, we have
\begin{align*}
\inf_{Q \in \Delta(\sR^n)}&\bigg\{\frac{\int(\eta-\beta)(r_n(\vg) - r_n(\vfstar_n))\mathrm{d}Q(\vg)}{\eta - 2\sigma^2\eta^2 - \beta - 2\sigma^2\beta^2} + \frac{\KL(Q||Q_{n,\beta,\alpha}) + 2\log\tfrac{1}{\delta}}{n(\eta - 2\sigma^2\eta^2 - \beta - 2\sigma^2\beta^2)}\bigg\}\\
&\leq \frac{\int(\eta-\beta)(r_n(\vg) - r_n(\vfstar_n))\mathrm{d}\wh Q(\vg)}{\eta - 2\sigma^2\eta^2 - \beta - 2\sigma^2\beta^2} + \frac{\KL(\wh Q||Q_{n,\beta,\alpha}) + 2\log\tfrac{1}{\delta}}{n(\eta - 2\sigma^2\eta^2 - \beta - 2\sigma^2\beta^2)}\\
&\leq \frac{2\gamma_n(2\eta\alpha, 2\beta\alpha) + \frac{1}{2\alpha}\|\fstar\|_{\gH}^2 + 2\log\tfrac{1}{\delta}}{n(\eta - 2\sigma^2\eta^2 - \beta - 2\sigma^2\beta^2)}\,.
\end{align*}
This concludes the proof.
\end{proofof}

\section{Bounds on the Relative Information Gain}

\subsection{Mercer Kernels}
\label{sec:mercer}

A positive-definite kernel $k$ on a set $\mathcal{X}$ is called a Mercer kernel if $\mathcal{X}$ is a compact metric space and the kernel function $k: \mathcal{X} \times \mathcal{X} \to \mathbb{R}$ is continuous. Mercer's theorem provides a useful representation for Mercer kernels. Let $\rho$ be a non-degenerate Borel measure on $\mathcal{X}$ and let $L^2(\mathcal{X}, \rho)$ denote the set of square integrable functions on $\mathcal{X}$. Namely,
\begin{equation*}
L^2(\mathcal{X}, \rho) := \left\{f:\mathcal{X} \to \mathbb{R} : \int_{\mathcal{X}}(f(x))^2\mathrm{d}\rho(x) < \infty\right\}\,.
\end{equation*}
Define the linear operator $L_k : L^2(\mathcal{X}, \rho) \to L^2(\mathcal{X}, \rho)$ as
\begin{equation*}
L_k(f)(x) := \int_{\mathcal{X}}k(x, y)f(y)\mathrm{d}\rho(y)\,.
\end{equation*}
\begin{theorem}[Mercer's Theorem]
If $k: \mathcal{X} \times \mathcal{X} \to \mathbb{R}$ is a Mercer kernel, then there exist non-negative eigenvalues $\xi_1 \geq \xi_2 \geq \cdots \geq 0$ and corresponding eigenfunctions $\phi_1, \phi_2, \dots$, such that
\begin{equation}
L_k(\phi_m) = \xi_m\phi_m, \quad \text{for all } m = 1, 2, \dots\,.
\end{equation}
In addition, the kernel function has the eigendecomposition
\begin{equation}
k(x, x^{\prime}) = \sum_{m=1}^{\infty}\xi_m\phi_m(x)\phi_m(x^{\prime})\,,
\end{equation}
where the convergence of the infinite series is absolute for each $x, x^{\prime} \in \mathcal{X}$ and uniform on $\mathcal{X} \times \mathcal{X}$.
\end{theorem}

\subsection{Auxiliary Lemmas}

Lemma \ref{lem:vakili1} and Lemma \ref{lem:vakili2} were extracted from the proof of Theorem 3 in \citet{vakili2021information}. Lemma \ref{lem:vakili1} allows us to re-write the information gain as the sum of a term depending on only $\mK_n^{\parallel}$ and another term depending on both $\mK_n^{\parallel}$ and $\mK_n^{\perp}$.
\begin{lemma}
For any $\eta > 0$ and any kernel $k$ that satisfies Assumption \ref{ass:mercer},
\begin{equation*}
\frac{1}{2}\log\det(\eta\mK_n + \mI) = \frac{1}{2}\log\det(\eta\mK_n^{\parallel} + \mI) + \frac{1}{2}\log\det(\mI + \eta(\mI + \eta\mK_n^{\parallel})^{-1}\mK_n^{\perp})\,.
\end{equation*}
\label{lem:vakili1}
\end{lemma}
Lemma \ref{lem:vakili2} provides an upper bound on the second term in Lemma \ref{lem:vakili1}.
\begin{lemma}
For any $\eta > 0$ and any kernel $k$ that satisfies Assumption \ref{ass:mercer},
\begin{equation*}
\frac{1}{2}\log\det(\mI + \eta(\mI + \eta\mK_n^{\parallel})^{-1}\mK_n^{\perp}) \leq \frac{1}{2}n\eta\delta_D\,.
\end{equation*}
\label{lem:vakili2}
\end{lemma}
To bound the first term on the right-hand side of the inequality in Lemma \ref{lem:vakili1}, we will consider the log-determinant of a $D \times D$ gram matrix. We define a $D$-dimensional feature map $\bs{\phi}_D(x) := [\phi_1(x), \dots, \phi_D(x)]^{\top}$, an $n \times D$ design matrix $\bs{\Phi}_{n,D} := [\bs{\phi}_D(x_1), \dots, \bs{\phi}_D(x_n)]^{\top}$ and a diagonal $D\times D$ matrix $\bs{\Xi}_D := \mathrm{diag}(\xi_1, \dots, \xi_D)$. Notice that $\mK_n^{\parallel} = \bs{\Phi}_{n,D}\bs{\Xi}_D\bs{\Phi}_{n,D}^{\top}$. We define the gram matrix
\begin{equation*}
\mG_n := \bs{\Xi}_D^{1/2}\bs{\Phi}_{n,D}^{\top}\bs{\Phi}_{n,D}\bs{\Xi}_D^{1/2}\,.
\end{equation*}
Since we will deal with the relative information gain, we will need to bound the log of the ratio of the determinants of two of these gram matrices. This is handled by Lemma
\begin{lemma}
For any $n \geq 1$, $D \geq 1$ and $\eta > \beta > 0$,
\begin{equation*}
\log\frac{\det(\eta\mG_n + \mI)}{\det(\beta\mG_n + \mI)} \leq D\log\frac{\eta}{\beta}\,.
\end{equation*}
\label{lem:gram_det}
\end{lemma}
\begin{proof}
Let $(\lambda_i)_{i=1}^{D}$ be the eigenvalues of $\mG_n$, which are all real and positive. We notice that for any $\zeta \in (0, 1]$,
\begin{equation*}
\det(\beta\mG_n + \zeta\mI) = \prod_{i=1}^{D}(\beta\lambda_i + \zeta) \leq \prod_{i=1}^{D}(\beta\lambda_i + 1) = \det(\beta\mG_n + \mI)\,.
\end{equation*}
Therefore, we can bound the logarithm of the ratio of determinants as
\begin{align*}
\log\frac{\det(\eta\mG_n + \mI)}{\det(\beta\mG_n + \mI)} &= D\log\tfrac{\eta}{\beta} + \log\frac{\det(\beta\mG_n + \tfrac{\beta}{\eta}\mI)}{\det(\beta\mG_n + \mI)}\\
&\leq D\log\tfrac{\eta}{\beta} + \log\frac{\det(\beta\mG_n + \mI)}{\det(\beta\mG_n + \mI)}\\
&= D\log\tfrac{\eta}{\beta}\,.
\end{align*}
This concludes the proof.
\end{proof}
We will use the fact that the second term on the right-hand side of the inequality in Lemma \ref{lem:vakili1} is non-negative.
\begin{lemma}
For any $\eta \geq 0$ and any kernel $k$ that satisfies Assumption \ref{ass:mercer},
\begin{equation*}
\frac{1}{2}\log\det(\mI + \eta(\mI + \eta\mK_n^{\parallel})^{-1}\mK_n^{\perp}) \geq 0\,.
\end{equation*}
\label{lem:non_neg_log_det}
\end{lemma}
\begin{proof}
Since $\mK_n = \mK_n^{\parallel} + \mK_n^{\perp}$, and both $\mK_n^{\parallel}$ and $\mK_n^{\perp}$ are positive semi-definite, it follows that
\begin{equation*}
\eta\mK_n + \mI \succcurlyeq \eta\mK_n^{\parallel} + \mI\,.
\end{equation*}
Due to monotonicity of the determinant with respect to the Loewner ordering (cf. Corollary 7.7.4 in \citealp{horn2012matrix}),
\begin{equation*}
\log\det(\eta\mK_n + \mI) \geq \log\det(\eta\mK_n^{\parallel} + \mI)\,.
\end{equation*}
Finally, using Lemma \ref{lem:vakili1},
\begin{equation*}
\frac{1}{2}\log\det(\mI + \eta(\mI + \eta\mK_n^{\parallel})^{-1}\mK_n^{\perp}) = \frac{1}{2}\log\det(\eta\mK_n + \mI) - \frac{1}{2}\log\det(\eta\mK_n^{\parallel} + \mI) \geq 0\,.
\end{equation*}
This concludes the proof.
\end{proof}
The next lemma comes from the proof of Corollary 1 in \citet{vakili2021information}. It provides bounds on $\delta_D$ under each of the eigenvalue decay assumptions.
\begin{lemma}
If the polynomial eigenvalue decay condition from Assumption \ref{ass:eig} is satisfied, then
\begin{equation*}
\delta_D \leq C_pD^{1-\beta_p}\psi^2\,.
\end{equation*}
If the exponential eigenvalue decay condition from Assumption \ref{ass:eig} is satisfied and $\beta_e = 1$, then
\begin{equation*}
\delta_D \leq \frac{C_{e,1}\psi^2}{C_{e,2}}\exp(-C_{e,2}D)\,.
\end{equation*}
If the exponential eigenvalue decay condition from Assumption \ref{ass:eig} is satisfied and $\beta_e \in (0,1)$, then
\begin{equation*}
\delta_D \leq \frac{2C_{1,e}\psi^2}{C_{e,2}\beta_e}\bigg(\frac{2}{C_{e,2}}\bigg(\frac{1}{\beta_e} - 1\bigg)\bigg)^{1/\beta_e-1}\exp\bigg(1 - \frac{1}{\beta_e}\bigg)\exp\bigg(-C_{e,2}\frac{D^{\beta_e}}{2}\bigg)\,.
\end{equation*}
\label{lem:delta_d}
\end{lemma}

\subsection{Proof of Proposition \ref{pro:rig_bound_trick}}
\label{sec:rig_bound_trick}
\begin{proofof}{Proposition \ref{pro:rig_bound_trick}}
Using Lemma \ref{lem:vakili1} and Lemma \ref{lem:non_neg_log_det}, we can re-write and then upper bound the relative information gain as
\begin{align*}
\gamma_n(\eta, \beta) &= \frac{1}{2}\log\frac{\det(\eta\mK_n^{\parallel} + \mI)}{\det(\beta\mK_n^{\parallel} + \mI)} + \frac{1}{2}\log\det(\mI + \eta(\mI + \eta\mK_n^{\parallel})^{-1}\mK_n^{\perp})\\
&- \frac{1}{2}\log\det(\mI + \beta(\mI + \beta\mK_n^{\parallel})^{-1}\mK_n^{\perp})\\
&\leq \frac{1}{2}\log\frac{\det(\eta\mK_n^{\parallel} + \mI)}{\det(\beta\mK_n^{\parallel} + \mI)} + \frac{1}{2}\log\det(\mI + \eta(\mI + \eta\mK_n^{\parallel})^{-1}\mK_n^{\perp})\,.
\end{align*}
By the Weinstein–Aronszajn identity,
\begin{equation*}
\frac{1}{2}\log\frac{\det(\eta\mK_n^{\parallel} + \mI)}{\det(\beta\mK_n^{\parallel} + \mI)} = \frac{1}{2}\log\frac{\det(\eta\mG_n + \mI)}{\det(\beta\mG_n + \mI)}\,.
\end{equation*}
Therefore, from Lemma \ref{lem:gram_det} and Lemma \ref{lem:vakili2}, it follows that
\begin{equation*}
\gamma_n(\eta, \beta) \leq \frac{1}{2}D\log\frac{\eta}{\beta} + \frac{1}{2}n\eta\delta_D\,.
\end{equation*}
This concludes the proof.
\end{proofof}

\subsection{Proof of Proposition \ref{pro:rel_info_gain}}
\label{sec:rig_rates}

\begin{proofof}{Proposition \ref{pro:rel_info_gain}}
From  Proposition \ref{pro:rig_bound_trick},
\begin{equation*}
\gamma_n(\eta, \beta) \leq \frac{1}{2}D\log\frac{\eta}{\beta} + \frac{1}{2}n\eta\delta_D\,.
\end{equation*}
We consider each eigenvalue decay condition separately. If the polynomial eigenvalue decay condition is satisfied, then Lemma \ref{lem:delta_d} tells us that
\begin{equation*}
\frac{1}{2}n\eta\delta_D \leq \frac{1}{2}n\eta C_pD^{1-\beta_p}\psi^2\,.
\end{equation*}
We choose the smallest value of $D$ such that
\begin{equation*}
n\eta C_pD^{1-\beta_p}\psi^2 \leq D\log\frac{\eta}{\beta}\,.
\end{equation*}
One can verify that this inequality is satisfied if we choose
\begin{equation*}
D = \big\lceil(n\eta C_p\psi^2)^{1/\beta_p}\log^{-1/\beta_p}(\tfrac{\eta}{\beta})\big\rceil\,.
\end{equation*}
With this choice of $D$, the relative information gain satisfies
\begin{equation*}
\gamma_n(\eta, \beta) \leq D\log\tfrac{\eta}{\beta} \leq (n\eta C_p\psi^2)^{1/\beta_p}\log^{1-1/\beta_p}(\tfrac{\eta}{\beta}) + \log\tfrac{\eta}{\beta}\,.
\end{equation*}
If the exponential eigenvalue decay condition is satisfied and $\beta_e = 1$, then Lemma \ref{lem:delta_d} tells us that
\begin{equation*}
\frac{1}{2}n\eta\delta_D \leq \frac{1}{2}n\eta \frac{C_{e,1}\psi^2}{C_{e,2}}\exp(-C_{e,2}D)\,.
\end{equation*}
This time, we choose
\begin{equation*}
D = \bigg\lceil \frac{1}{C_{e,2}}\log\frac{n\eta C_{e,1}\psi^2}{C_{e,2}}\bigg\rceil\,.
\end{equation*}
With this choice of $D$, the relative information gain satisfies
\begin{equation*}
\gamma_n(\eta, \beta) \leq \frac{1}{2}\frac{1}{C_{e,2}}\log\frac{n\eta C_{e,1}\psi^2}{C_{e,2}}\log\frac{\eta}{\beta} + \frac{1}{2}\log\frac{e\eta}{\beta}\,.
\end{equation*}
If the exponential eigenvalue decay condition is satisfied and $\beta_e = 1$, then Lemma \ref{lem:delta_d} tells us that
\begin{equation*}
\frac{1}{2}n\eta\delta_D \leq \frac{1}{2}n\eta \frac{2C_{1,e}\psi^2}{C_{e,2}\beta_e}\bigg(\frac{2}{C_{e,2}}\bigg(\frac{1}{\beta_e} - 1\bigg)\bigg)^{1/\beta_e-1}\exp\bigg(1 - \frac{1}{\beta_e}\bigg)\exp\bigg(-C_{e,2}\frac{D^{\beta_e}}{2}\bigg)\,.
\end{equation*}
If we choose
\begin{equation*}
D = \bigg\lceil \bigg(\frac{2}{C_{e,2}}\log\bigg(\frac{2n\eta C_{1,e}\psi^2}{C_{e,2}\beta_e}\bigg(\frac{2}{C_{e,2}}\bigg(\frac{1}{\beta_e} - 1\bigg)\bigg)^{1/\beta_e-1}\exp\bigg(1 - \frac{1}{\beta_e}\bigg)\bigg)\bigg)^{1/\beta_e}\bigg\rceil\,,
\end{equation*}
Then the relative information gain satisfies
\begin{align*}
\gamma_n(\eta,\beta) &\leq \frac{1}{2}\left(\frac{2}{C_{e,2}}\log\bigg(\frac{2n\eta C_{1,e}\psi^2}{C_{e,2}\beta_e}\bigg(\frac{2}{C_{e,2}}\bigg(\frac{1}{\beta_e} - 1\bigg)\bigg)^{1/\beta_e-1}\exp\bigg(1 - \frac{1}{\beta_e}\bigg)\bigg)\right)^{1/\beta_e}\log\frac{\eta}{\beta}\\
&+ \frac{1}{2}\log\frac{e\eta}{\beta}\,.
\end{align*}
Therefore, if the exponential decay condition is satisfied with $\beta_e \in (0, 1]$, the relative information gain satisfies
\begin{equation*}
\gamma_n(\eta,\beta) \leq \frac{1}{2}\left(\frac{2}{C_{e,2}}\log(n\eta C_{\beta_e})\right)^{1/\beta_e}\log\frac{\eta}{\beta} + \frac{1}{2}\log\frac{e\eta}{\beta}\,,
\end{equation*}
where $C_{\beta_e} = \frac{C_{e,1}\psi^2}{C_{e,2}}$ if $\beta_e = 1$, and $C_{\beta_e} = \frac{2C_{e,1}\psi^2}{C_{e,2}\beta_e}(\frac{2 - 2\beta_e}{C_{e,2}\beta_e})^{1/\beta_e-1}\exp(\frac{1-\beta_e}{\beta_e})$ if $\beta_e \in (0, 1)$.
\end{proofof}

\section{Rates of Convergence}

\subsection{Proof of Corollary \ref{cor:rate_poly}}
\label{sec:app_rate_poly}

\begin{proofof}{Corollary \ref{cor:rate_poly}}
When the polynomial eigenvalue decay condition is satisfied, the bound on the relative information gain in Proposition \ref{pro:rel_info_gain} reads as 
\begin{equation*}
\gamma_n(2\eta\alpha, 2\beta\alpha) \leq (2n\eta\alpha C_p\psi^2)^{1/\beta_p}\log^{1-1/\beta_p}(\tfrac{\eta}{\beta}) + \log\tfrac{\eta}{\beta}\,.
\end{equation*}
From this and Theorem \ref{thm:rig_excess_risk}, we know that (with probability at least $1 - \delta$), the excess risk of $Q_{n,\eta,\alpha}$ satisfies
\begin{equation*}
\int R_n(\vg)\mathrm{d}Q_{n, \eta, \alpha}(\vg) \leq \frac{2(2n\eta\alpha C_p\psi^2)^{1/\beta_p}\log^{1-1/\beta_p}(\tfrac{\eta}{\beta}) + 2\log\tfrac{\eta}{\beta} + \frac{1}{2\alpha}\|\fstar\|_{\gH}^2 + 2\log\tfrac{1}{\delta}}{n(\eta - 2\sigma^2\eta^2 - \beta - 2\sigma^2\beta^2)}\,.
\end{equation*}
If we set $\eta = \frac{1}{4\sigma^2}$ and $\beta = \frac{1}{32\sigma^2}$, then this inequality becomes
\begin{equation*}
\int R_n(\vg)\mathrm{d}Q_{n, \eta, \alpha}(\vg) \leq \frac{512\sigma^2}{47n}\bigg(2\bigg(\frac{n\alpha C_p\psi^2}{2\sigma^2}\bigg)^{1/\beta_p}\log^{1-1/\beta_p}(8) + \frac{1}{2\alpha}\|\fstar\|_{\gH}^{2} + 2\log\frac{8}{\delta}\bigg)\,.
\end{equation*}
If we then set $\alpha = n^{-\frac{1}{1 + \beta_p}}$, we obtain
\begin{align*}
\int R_n(\vg)\mathrm{d}Q_{n, \eta, \alpha}(\vg) &\leq \frac{512\sigma^2}{47}\bigg(2\bigg(\frac{C_p\psi^2}{2\sigma^2}\bigg)^{1/\beta_p}\log^{1-1/\beta_p}(8) + \frac{1}{2}\|\fstar\|_{\gH}^{2}\bigg)n^{-\frac{\beta_p}{1 + \beta_p}}\\
&+ \frac{1024\sigma^2\log\frac{8}{\delta}}{47n}\,,
\end{align*}
and the proof is complete.
\end{proofof}
\noindent One can instead choose a value of $\alpha$ that results in the best dependence on $\sigma$ and $\|\fstar\|_{\gH}$ (assuming $\|\fstar\|_{\gH}$ is known). If
\begin{equation*}
\alpha = \big(\tfrac{\beta_p}{2}\big)^{\frac{\beta_p}{1 + \beta_p}}\big(\frac{C_p\psi^2}{2}\big)^{-\frac{1}{1 + \beta_p}}\log^{\frac{1-\beta_p}{1+\beta_p}}(8)\|\fstar\|_{\gH}^{\frac{2\beta_p}{1+\beta_p}}\sigma^{\frac{2}{1+\beta_p}}n^{-\frac{1}{1+\beta_p}}\,,
\end{equation*}
then the excess risk bound becomes
\begin{equation*}
\int R_n(\vg)\mathrm{d}Q_{n, \eta. \alpha}(\vg) \leq \tfrac{512}{47}\big(\tfrac{1+\beta_p}{\beta_p}\big(\tfrac{\beta_pC_p\psi^2}{4}\big)^{\frac{1}{1+\beta_p}}\log^{\frac{\beta_p-1}{1 + \beta_p}}(8)\|\fstar\|_{\gH}^{\frac{2}{1+\beta_p}}\sigma^{\frac{2\beta_p}{1 + \beta_p}}n^{-\frac{\beta_p}{1 + \beta_p}} + \tfrac{2\sigma^2\log\tfrac{8}{\delta}}{n}\big)\,.
\end{equation*}

\subsection{Proof of Corollary \ref{cor:rate_exp}}
\label{sec:app_rate_exp}

\begin{proofof}{Corollary \ref{cor:rate_exp}}
When the exponential eigenvalue decay condition is satisfied, the bound on the relative information gain in Proposition \ref{pro:rel_info_gain} reads as
\begin{equation*}
\gamma_n(2\eta\alpha,2\beta\alpha) \leq \frac{1}{2}\left(\frac{2}{C_{e,2}}\log(2n\eta\alpha C_{\beta_e})\right)^{1/\beta_e}\log\frac{\eta}{\beta} + \frac{1}{2}\log\frac{e\eta}{\beta}\,,
\end{equation*}
From this and Theorem \ref{thm:rig_excess_risk}, we know that (with probability at least $1 - \delta$), the excess risk of $Q_{n,\eta,\alpha}$ satisfies
\begin{equation*}
\int R_n(\vg)\mathrm{d}Q_{n, \eta, \alpha}(\vg) \leq \frac{\big(\frac{2}{C_{e,2}}\log(2n\eta\alpha C_{\beta_e})\big)^{1/\beta_e}\log\frac{\eta}{\beta} + \log\frac{e\eta}{\beta} + \frac{1}{2\alpha}\|\fstar\|_{\gH}^2 + 2\log\tfrac{1}{\delta}}{n(\eta - 2\sigma^2\eta^2 - \beta - 2\sigma^2\beta^2)}\,.
\end{equation*}
If we set $\eta = \frac{1}{4\sigma^2}$ and $\beta = \frac{1}{32\sigma^2}$ and $\alpha = 1$, then this inequality becomes
\begin{equation*}
\int R_n(\vg)\mathrm{d}Q_{n, \eta, \alpha}(\vg) \leq \frac{512\sigma^2}{47n}\bigg(\bigg(\frac{2}{C_{e,2}}\log\bigg(\frac{n C_{\beta_e}}{2\sigma^2}\bigg)\bigg)^{1/\beta_e}\log(8) + \frac{1}{2}\|\fstar\|_{\gH}^2 + 2\log\frac{\sqrt{8e}}{\delta}\bigg)\,,
\end{equation*}
and the proof is complete.
\end{proofof}
\noindent If $\|\fstar\|_{\gH}$ is known, one can instead set $\alpha = \|\fstar\|_{\gH}^{2}$ to obtain an excess risk bound with polylogarithmic dependence on the norm of $\fstar$.

\end{document}